%% file: main.tex
\documentclass[11pt, a4paper, twocolumn]{deepmind}

\makeatletter
\renewcommand\bibentry[1]{\nocite{#1}{\frenchspacing\@nameuse{BR@r@#1\@extra@b@citeb}}}
\makeatother

\usepackage{kantlipsum, lipsum}
\usepackage{dsfont}
\usepackage{subfigure}
\usepackage{hyperref}
\usepackage{booktabs} 

\usepackage{amsmath}
\usepackage{amssymb}
\usepackage{mathtools}
\usepackage{amsthm}
\RequirePackage{algorithm}
\RequirePackage{algorithmic}
\usepackage[capitalize,noabbrev]{cleveref}
\input{math_commands.tex}
\theoremstyle{plain}
\newtheorem{theorem}{Theorem}[section]
\newtheorem{proposition}[theorem]{Proposition}

\theoremstyle{definition}

\theoremstyle{remark}

\renewcommand{\today}{} 
\usepackage[dvipsnames]{xcolor}

\definecolor{yellow}{rgb}{0.75,0.65,0.15}

\def\yellow#1{\textcolor{yellow}{\textbf{#1}}}
\definecolor{turquoise}{rgb}{0.3,0.7,0.7}
\def\turquoise#1{\textcolor{turquoise}{\textbf{#1}}}
\definecolor{dark turquoise}{HTML}{01665e}
\definecolor{brown}{rgb}{0.6,0.3,0.}
\def\brown#1{\textcolor{brown}{\textbf{#1}}}
\definecolor{light brown}{HTML}{bf812d}
\usepackage{microtype}
\usepackage{graphicx}
\usepackage[dvipsnames]{xcolor}

\usepackage[authoryear, sort&compress, round]{natbib}

\graphicspath{{figures/}}

\title{Selective Credit Assignment}

\correspondingauthor{veronica.chelu@mail.mcgill.ca}

\author[1]{Veronica Chelu}
\author[2]{Diana Borsa}
\author[1,2,3]{Doina Precup}
\author[2]{Hado van Hasselt}

\affil[1]{McGill University, Mila}
\affil[2]{DeepMind}
\affil[3]{Canada CIFAR AI Chair}

\begin{abstract}
Efficient credit assignment is essential for reinforcement learning algorithms in both prediction and control settings. We describe a unified view on temporal-difference algorithms for \emph{selective credit assignment}. These selective algorithms apply weightings to quantify the contribution of learning updates. We present insights into applying weightings to value-based learning and planning algorithms, and describe their role in mediating the backward credit distribution in prediction and control. Within this space, we identify some existing online learning algorithms that can assign credit selectively as special cases, as well as add new algorithms that assign credit backward in time counterfactually, allowing credit to be assigned off-trajectory and off-policy.
\end{abstract}

\begin{document}
\maketitle

\newcommand{\expect}[2]{\mathds{E}_{{#1}} \left[ {#2} \right]}
\newcommand{\myvec}[1]{\boldsymbol{#1}}
\newcommand{\myvecsym}[1]{\boldsymbol{#1}}
\newcommand{\vx}{\myvec{x}}
\newcommand{\vy}{\myvec{y}}
\newcommand{\vz}{\myvec{z}}
\newcommand{\vtheta}{\myvecsym{\theta}}

\section{Introduction}
\label{introduction}
In reinforcement learning (RL)  \citep{sutton2018} an agent must assign credit or blame for the rewards it obtains to past states and actions. This problem is difficult because rewards may be sparse and may occur much later than the events that helped cause them. Moreover, the agent's observations are typically noisy or aliased, further complicating its reasoning about the root causes of observed rewards. Effective credit assignment across long stretches of time in complex environments remains a largely unsolved and actively pursued research problem \citep{Harutyunyan2019HindsightCA, Hung2019OptimizingAB, ArjonaMedina2019RUDDERRD, Ke2018SparseAB, Mesnard2020CounterfactualCA, chelu2020forethought}. 
 
In this paper, we describe a generic way of adding selectivity in online credit assignment, leading to a unified view of the space of algorithms available. We present insights into the effect of  non-uniformly weighting the learning updates of value-based algorithms to improve credit assignment. 
  
As an example, consider Fig.~\ref{fig:ms_pacman_intro}, which contains results of agents playing the Atari game of Ms.Pac-Man.  The baseline performance (\brown{brown}) is due to an agent that updates its action values for each transition using a form of Q-learning \citep{watkins1992q} with expected eligibility traces \citep{van2020expected} (algorithm details are in later sections). We then consider a modified version of the game, where some of the observations are very noisy---the idea is that this is similar to a hardware camera on a robot that occasionally adds substantial noise, for instance due to a faulty cable.  The same algorithm performs far worse on this new version of the game (\yellow{yellow} line at the bottom).
However, we recover the baseline performance if we weight the updates (and appropriately modify the algorithm), despite the impoverished input signal (\turquoise{turquoise} line).

\begin{figure}[t!]
\begin{center}
      \includegraphics[scale=.4]{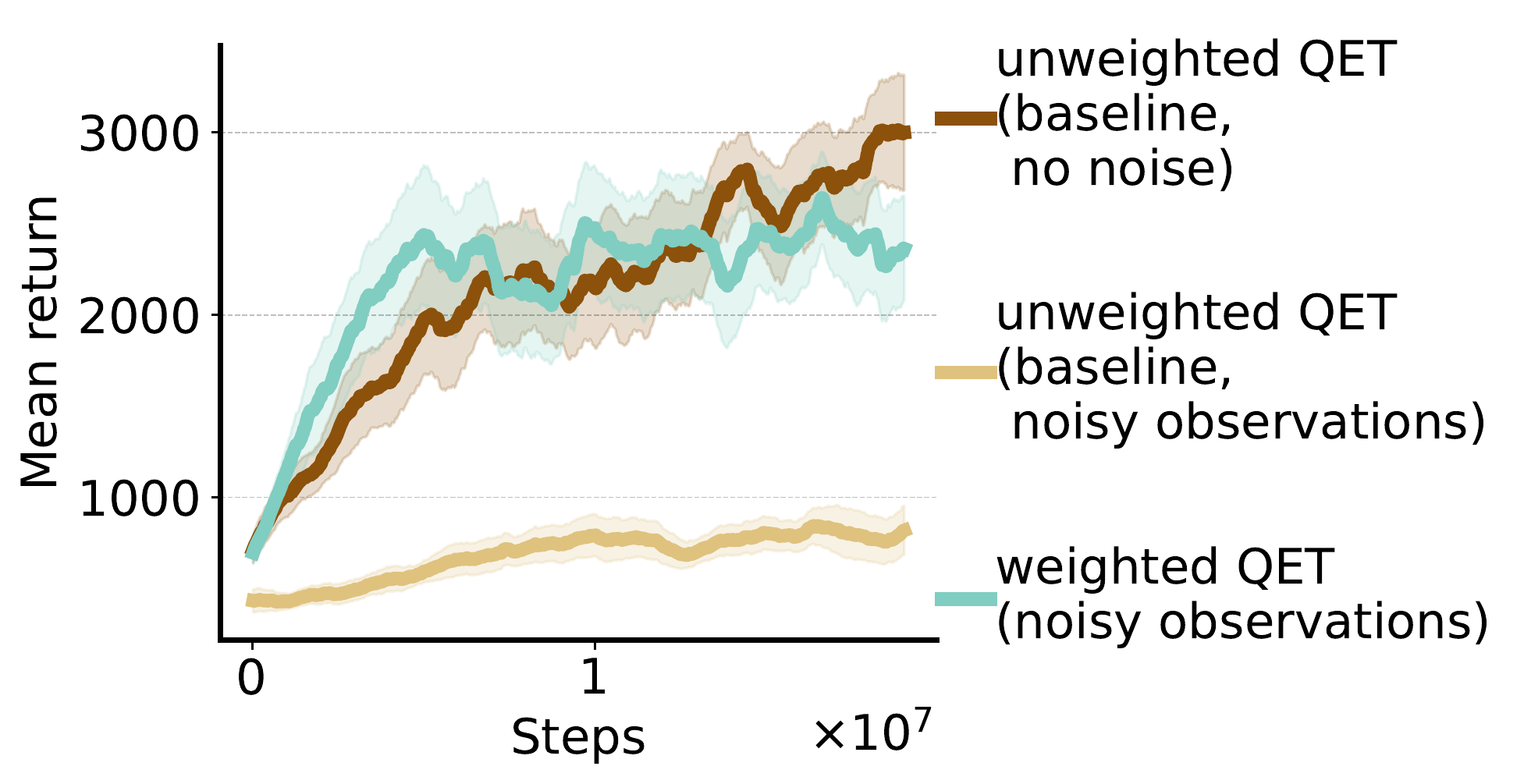}
    \end{center}
\caption{\label{fig:ms_pacman_intro}
\small \textbf{Atari Ms.Pac-Man:}
The curves show mean returns on Ms. Pac-Man for Q-learning with expected eligibility traces (QET), in different settings and with different weightings. The \emph{baseline} QET algorithm updates its values using a uniform weight on each transition. In the \emph{noisy} runs, Gaussian noise is added to the observations with probability $\epsilon=0.5$. The \emph{weighted} algorithm down-weights transitions with noisy inputs, and adapts the trace parameter $\lambda$ to be aware of these non-uniform weightings. Shaded areas denote standard error over $5$ seeds. These experiments demonstrate that uniform weightings fail in the noisy setting, and performance can be recovered with more careful weighting.
}
\end{figure}

The Q-learning algorithms illustrated in Fig.~\ref{fig:ms_pacman_intro} use a form of \emph{temporal difference (TD) learning}~\citep{Sutton:1988} to learn predictions online from sampled experience by \emph{bootstrapping} on other predictions \citep{sutton2018}. In such algorithms, credit is assigned online via an eligibility trace \citep{Sutton:1988, Peng1996IncrementalMQ, Hasselt2015LearningTP, sutton2018}, which keeps track of how the parameters of the value function estimator should be adjusted when a prediction error occurs to appropriately correct predictions made on earlier time steps. Canonical TD algorithms use recency and frequency to determine appropriate credit assignment \citep{Sutton1984TemporalCA}; frequent states are updated more, and preceding states that are temporally closer to a prediction error receive more blame for this error. 
 
Our \textbf{first contribution} is an analysis of stability of arbitrary weightings for on-policy algorithms, showing how to ensure stable convergent learning. Previous work has mainly focused on weightings for the off-policy case, or overlooked stability, perhaps because it is not broadly acknowledged that even on-policy algorithms can diverge with a non-uniform weighting. We discuss a simple weighting that ensures convergence and discuss how this on-policy weighting is connected to the idea of emphasis \citep{sutton2016emphatic} which was proposed to stabilize off-policy TD.
 
\textbf{Second}, we provide examples of weightings that greatly enhance learning, e.g., in the face of noisy observations, and show that the analysis mentioned above informs concrete algorithms with improved performance. 

In RL, we have \emph{learning} algorithms which use experiential data, and \emph{planning} algorithms which are driven by a model. We consider the online planning algorithm proposed by \citet{hasselt2020}, using \emph{on-policy expected eligibility traces} (ET) in place of the instantaneous traces of standard TD, to propagate credit backward not just to the states that occur on the current trajectory but also to other \emph{possible} trajectories leading to the current state. This allows counterfactual credit assignment to states that \emph{could} have happened, but did not occur recently on the current trajectory.
Our \textbf{third contribution}, is to provide insights on using selectivity with planning algorithms based on expected traces to improve stability, data efficiency and knowledge reuse. Specifically, we show how selectivity can be used to learn \emph{off-policy expected traces}, online from a single stream of experience and, separately, how it can be used for \emph{sparse credit assignment} which of interest for hierarchical learning. We provide concrete examples of weightings for expected traces that enhance planning, e.g., in the face of noisy observations.  As an \textbf{additional contribution}, we provide a more computationally efficient version of the value-based control planning algorithm QET \citep{hasselt2020} (the planning counterpart of the learning algorithm Q($\lambda$)), saving a factor of $|\mathcal{A}|$ --- the number of actions, true for all implementations, which can be significant in practice.
\section{Background and preliminaries}
\label{background and preliminaries}
We denote random variables with uppercase (e.g., $S$) and the obtained values with lowercase letters (e.g., $S = s$). Multi-dimensional functions or vectors are bolded (e.g., $\b$), as are matrices (e.g., $\A$). For state-dependent functions, we also allow time-dependent shorthands (e.g., $\gamma_t\!=\! \gamma(S_t)$).


\subsection{Reinforcement learning problem setup}
We consider the usual RL setting of an agent interacting with an environment, modelled as an infinite horizon \emph{Markov Decision Process} (MDP) $(\calS, \calA, P, r)$, with a finite state space $\calS$, a finite action space $\calA$, a state-transition distribution $P\!:\! \calS \!\times \!\calA \to \calP(\calS)$---with $\calP(\calS)$ the set of probability distributions on $\calS$ and $P(s^\prime|s, a)$ the probability of transitioning to state $s^\prime$ from $s$ by choosing action $a$, and a reward function $r : \calS \times \calA \rightarrow \R$. A policy $\pi : \calS \rightarrow \calP(\calA)$ maps states to distributions over actions; $\pi(a|s)$ is the probability of choosing action $a$ in state $s$ and $\pi(s)$ is the probability distribution of actions in state $s$. Let $S_t, A_t, R_t$ denote the random variables of state, action and reward at time $t$, respectively.

The goal of \emph{policy evaluation} is to estimate the \emph{value function} $V_\pi$, defined as the expectation of the discounted return under policy $\pi$:
\begin{align}
G_t &= R_{t+1} + \gamma_{t+1} G_{t+1} \nonumber\,,\\
\V_{\pi}(s) & \equiv \mathbb{E}_\pi [ G_t \mid S_t=s] \,,\label{eq:learning_target}
\end{align} 
where $\gamma : \calS \to [0, 1]$ is a discount factor and $\mathbb{E}_\pi[\cdot]$ denotes the expectation over trajectories sampled under $\pi$. 
In the function approximation setting, we update the parameters $\w$ of a function $\V_{\w}$ to estimate $\V_\pi$. For any $f$, we use $\nabla f_{\w_t}$ as shorthand for the gradient $\nabla_{\w}f_{\w}$ evaluated at $\w=\w_t$. 

For \emph{off-policy} policy evaluation,
the goal is to estimate $V_\pi$ whilst interacting with the MDP by sampling actions according to a different behaviour policy $\mu$.
In \emph{control}, the learner's goal is to find a policy $\pi$ that maximizes the value $\V$. Value-based methods for control \citep[e.g., Q-learning;][]{watkins1992q} use state-action value functions to learn implicit (e.g., greedy) policies.

\subsection{Online credit assignment}
We start with credit assignment algorithms used for learning value functions of a given policy---the policy evaluation setting---after which we look at methods that adapt to maximize performance---the control setting.

\subsubsection{Learning and planning for evaluation}
\paragraph{TD($\lambda$)}{
A popular and effective algorithm to learn $\V_{\w} \approx \V_{\pi}$ online and on-trajectory, is TD($\lambda$) \cite{Sutton:1988}:
\begin{align}
    \w_{t+1} &= \w_t + \alpha^{\w}_t \Delta_t^\V\,, \text{ with } \Delta^\V_t \equiv \e_t \delta_t\,, \text{ and}    \label{eq:TD_lambda_update}
    \\
    \e_t &= \gamma_t \lambda_t {\e}_{t-1} + \nabla\V_{\w_t}(S_t) \,,
\end{align}
where $\w \in \mathbb{R}^d$ are parameters of $\V_{\w}$ to be updated, $\Delta^\V_t = \e_t \delta_t$ is an update with \emph{TD error} $\delta_t = R_{t+1} + \gamma_t \V_{\w_t}(S_{t+1}) - \V_{\w_t}(S_t)$ and \emph{accumulating eligibility trace} $\e_t$, and $\alpha^{\w}_t\in (0,1)$ is a (possibly time-varying) step-size parameter.  For instance, $\w_t$ could be the weights of a neural network, or of a linear function of a feature mapping $\x(s)$, s.t. $\V_{\w}(s) = \w^\top \x(s)$.
The trace-decay parameter $\lambda_t \in [0, 1]$ interpolates between one-step TD learning and Monte-Carlo methods.  Several variations exist \citep[e.g.,][]{Maei:2011,Sutton:2014,vanSeijen:2014,Hasselt2015LearningTP}.  For clarity, we focus on the canonical variant above.
}

\paragraph{ETD($\lambda$)}{
The ETD($\lambda$) algorithm  \cite{sutton2016emphatic} was introduced for correcting TD algorithms when learning is off-policy, i.e. when the learning distribution differs from the sampling distribution due to a discrepancy between the behaviour and the target policy. This solution involves weighting the trace using a history-dependent function:
\begin{align}
\e_t &= \gamma_t \lambda_t \e_{t-1} + \rho_t M_t \nabla \V_{\w_t}(S_t)\label{eq:offpolicy_weighting_mixture_MC}\,, \text{ with } \\
M_t &= \lambda_t i_t + (1-\lambda_t) F_t \nonumber\,,\text{ and } 
F_t = \gamma_t \rho_{t-1} F_{t-1} + i_t\,,
\end{align}
where $M$ is the \emph{emphasis}, $F$ is the \emph{follow-on trace}, $i$ is a non-negative arbirary \emph{interest} function, originally introduced to focus learning \cite{Mahmood2015EmphaticTL, sutton2016emphatic}, and $\rho_t =\pi(A_t|S_t)/\mu(A_t|S_t)$ is an importance sampling ratio between the target policy $\pi$ and the behaviour policy $\mu$. ETD algorithms optimize the \emph{excursion} objective, defined as the value error under the stationary distribution of the behaviour policy $d_{\mu}$. \citet{zhang2020learning} consider learning the expectation of the follow-on trace $f(s) = \E_\mu[F_t|S_t=s]$, and using it directly in place of the history-dependent weighting. Similarly, \citet{ray2021} use an expectation of the $n$-step follow-on trace.
}
\paragraph{ET($\lambda$)}{
\emph{Expected eligibility trace} (ET) algorithms \citep{hasselt2020} have been introduced for \emph{off-trajectory}, \emph{on-policy} value learning, replacing the instantaneous trace ${\e}_t$ in Eq.~\eqref{eq:TD_lambda_update} with an estimated expectation:
\begin{align}
\z_{\T}(s) & \approx \E_{\pi}\left[{\e}_{t} \mid S_{t} = s\right]\,. \label{eq:EET}
\end{align}
Expected traces can be thought of as a true expectation model, and approximations thereof can be learned by regressing on the instantaneous eligibility trace $\e_t$:
\begin{align}
\T_{t+1} &= \T_{t} + \alpha^{\T}_t \Delta^{\z}_t\,,\text{ with}\label{eq:learning_traces}
\\
\Delta^{\z}_t &\equiv\frac{\partial {\z}_{\T_t}(S_{t})}{\partial \T_t}\left(\e_t -{\z}_{\T_t}(S_{t})\right)\nonumber\,,
\end{align}
with step-size parameter $\alpha^{\T}_t$. Expected traces can also be learned by \emph{backward} or \emph{time-reversed} TD, leading to multi-step updates similar to the TD($\lambda$) version of TD (see \cref{A:Expected eligibility traces} for details).
}

\subsubsection{Learning and planning for control}
For value-based control, we consider multi-step value-based analogs of TD($\lambda$) and ET($\lambda$). The goal is to learn action-values $\Q_{\w}(s, a)$, rather than state-values $\V_{\w}(s)$ we used thus far in our exposition, so that we can then derive a greedy policy $\arg\max_a \Q_{\w}(s, a)$ with respect to those values. 

\paragraph{Q($\lambda$)}{
The Q($\lambda$) algorithm \citep{Peng1996IncrementalMQ} is the analog of TD($\lambda$) for control. With $\w$ now representing the parameters of the action-value function $\Q_\w$, the corresponding action-value update replaces $\Delta^\V_t$ in Eq. \eqref{eq:TD_lambda_update} with
 \begin{align}
\Delta^\Q_t &\equiv \left(G^{\lambda}_{t} -  \Q_{\w_t}(S_t, A_t)\right)\nabla \Q_\w(S_t, A_t)\,,\text{ with}  \nonumber
\\
G^{\lambda}_t &= R_{t+1} + \gamma_{t+1} (1 - \lambda_{t+1})\max_a \Q_{\w}(S_{t+1}, a) \nonumber
\\
& \qquad + \gamma_{t+1} \lambda_{t+1} G^{\lambda}_{t+1}\nonumber \,.
\end{align}
Standard Q-learning corresponds to $\lambda_t=0, \forall t$.
An alternative derivation of the Q($\lambda$) algorithm (see \cref{A:New derivation of Off-policy}), yields:
\begin{align}
\Delta^\Q_t &\!=\! \e_{t-1} R^{\lambda}_t \!- \!\Q_{\w_t}(S_t, A_t) \nabla \Q_{\w_t}(S_t, A_t) \!\,,\! \text{ with}\! \label{eq:delta_q}
\\
R^{\lambda}_t &\!=\! R_{t+1} + \gamma_{t+1} (1-\lambda_{t+1}) \max_{a} \Q_{\w_t}(S_{t+1}, a)\,,\text{ and} \nonumber
\\
\e_t &\!=\! \gamma_t\lambda_t \e_{t-1} +\nabla \Q_{\w_t}(S_t, A_t)
\nonumber\,,
\end{align}
where $R^{\lambda}$ is a \emph{fixed-horizon one-step target}, composed of the reward at the next timestep and the bootstrapped value using a modified discount $\gamma_{t+1}(1-\lambda_{t+1})$, $\e_t$ is an accumulating eligibility trace, and $\alpha^{\w}_t$ is a step-size parameter.
}

\paragraph{QET($\lambda$)}{The QET($\lambda$) algorithm \citep{hasselt2020} is analogous to Q($\lambda$), but using learned models of expected traces in place of the standard instantaneous traces.
}

\section{Selective credit assignment}
All the aforementioned algorithms apply some sort of selection mechanism over which samples are used for learning. This choice can be implicit---e.g., in TD($\lambda$) it is determined by the behaviour distribution, or explicit---e.g., in ETD($\lambda$).

We take a unifying view over all the previous algorithms by separating the learning algorithm from this selection mechanism, now captured explicitly in a function $\omega : \S \to [0, \infty)$.
We will mostly consider selectivity as a function of state; other extensions are possible, such as allowing it to be a function of history $\omega : \calH \to [0, \infty)$, with $\calH$ the space of histories, or a function over some feature representation.

\paragraph{TD($\lambda, \omega$)}{We call \emph{selective TD($\lambda, \omega$)} the generic algorithm that uses a weighting function for the value update, replacing the standard trace with a \emph{weighted} or \emph{selective eligibility trace} (henceforth using the $\ \tilde{ }\ $ superscript to denote explicit selectivity):
\eqq{
\tilde{\e}_{t} \!=\! 
\gamma_t \lambda_t \tilde{\e}_{t-1} + \omega(S_t) \nabla \V_{\w_t}(S_{t}) \label{eq:weighted_eligibility_trace}\,.
}
Generally, the weighting function controls how much credit is received by a state. For instance, if $\omega(s) = 0$, contributions in the trace from $s$ are dropped. Hence, predictions at $s$ are then not corrected, and are learned solely via generalization, if at all. The backward view \eqref{eq:weighted_eligibility_trace} can be equivalently expressed in a forward view as
\begin{align}
\tilde{\Delta}_t^{\V} &= \omega_t (G^{\lambda}_t - \V_{\w_{t}}(S_t))\nabla \V_{\w_t}\,, \text{with } \label{eq:weighted_TD_forward_update}
\\
G^{\lambda}_t &= R_{t+1} + \gamma_{t+1}(1-\lambda_{t+1}) \V_{\w_{t}}(S_{t+1}) + \gamma_{t+1}\lambda_{t+1} G^{\lambda}_{t+1}\nonumber\,.
\end{align}
Fig.~\ref{fig:toy_aliasing} (Left) illustrates how state weightings can impact value learning. 
}

\begin{figure}[t!]
\begin{center}
        \raisebox{-0.5\height}{\includegraphics[scale=.29]{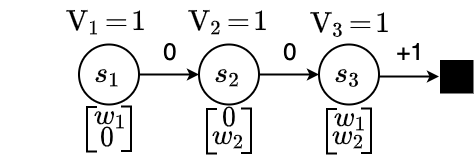}}
        \hspace*{.2in}
    \raisebox{-0.5\height}{\includegraphics[scale=.29]{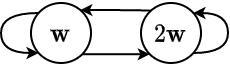}}
    \end{center}
\caption{\label{fig:toy_aliasing}
\small \textbf{(Left)} The value function $\V_{\w}(s) = \w^\top \x(s)$ has two parameters $\w = [w_1, w_2]^\top$. The features are $\x(s_1) = [1, 0]^\top$, $\x(s_2) = [0, 1]^\top$, $\x(s_3) = [1, 1]^\top$.  The true values are $\V(s_1) = \V(s_2) = \V(s_3) = 1$, but cannot be represented accurately in all states at the same time. Uniformly weighting the state space can result in trade-offs, such as $w_1 = w_2 = 2/3$, s.t. $\V(s_1) = \V(s_2) = 2/3$ and $\V(s_3) = 4/3$.  When desired, weighting selectively can add flexibility, for instance by ignoring $s_1$ and learning the correct value function only in states $s_2$ and $s_3$ (by setting $w_1=0$, $w_2=1$). We could also reweigh states without ignoring them completely, by having low but non-zero weight on them.
\textbf{(Right)} Two-state MDP example from \citet{sutton2016emphatic}. Consider learning on-policy with a uniform policy and transition dynamics. There is only one feature and one parameter $w$.}
\end{figure}

It is well known that in the \emph{off-policy} setting, discrepancies between the behaviour policy used to sample experience and the target policy whose value is being learned can destabilize learning, even with posterior corrections in the form of importance sampling weights \cite{Tsitsiklis_VanRoy:97}. Moreover, even when off-policy TD does converge, its solution may be arbitrarily far from the optimal representable value \cite{Kolter2011TheFP}.

\subsection{On-policy TD can diverge}{
Semi-gradient TD algorithms are known to converge in the \emph{on-policy} setting, as long as experience is sampled uniformly \citep{Tsitsiklis_VanRoy:97, sutton2018}.
However, divergence can still happen \emph{on-policy} when not carefully considering how experience is weighted. 
Naively using an arbitrary weighting $\omega$ in $\TD(\lambda, \omega)$ can lead to divergence.  Consider using TD($\lambda, \omega$) on the problem in Fig.~\ref{fig:toy_aliasing} (Right), in which we weight experience non-uniformly by putting a non-zero weight only on the first state ($\omega(s) = 1$) and using a constant trace-decay $\lambda(s) = 0, \forall s$. If $\gamma>0.5$ and, initially, $w_0 = 1$, then $w$ increases without bounds because we repeatedly update for the transition $w\to 2w$, while ignoring the transition $2w \to w$, causing divergence \citep{sutton2018, hasselt2018}.

Similarly, on-policy TD($\lambda, \omega(\cdot)=1$) with uniform state weighting and \emph{non-uniform bootstrapping} (e.g., $\lambda(s)$ varies across states) can diverge \citep{white2017unifying}. In Fig.~\ref{fig:toy_aliasing} (Right), this occurs if the weighting is non-zero in both states but we bootstrap only on the second state. In the next section we describe several algorithms that can address this issue.
}

\subsection{Selectivity through emphasis for stability} \label{Selectivity through emphasis for stability}
The aforementioned issues can arise when there is an imbalance between how often a state is used to compute update targets, and how often it is updated itself. Two kinds of weightings have been proposed to correct for this imbalance---\emph{emphasis} and \emph{distribution ratios}. Both can be seen as instances of TD($\lambda, \omega$) for different weighting functions $\omega$. 
Because distribution ratios are generally difficult to estimate, we focus on emphatic weightings. We previously described the emphatic algorithm ETD($\lambda$), which weights experience using a history-dependent weighting $
\omega_t =\rho_t M_t
$.

A different instance of TD($\lambda, \omega$) can be obtained by learning the expectation of the follow-on weighting $f(s) = \E_\mu[F_t|S_t=s]$, similar to \citet{zhang2020learning} and \citet{ray2021}:\!
\begin{align}
\omega_t &\!=\! \rho_t m_t\! \,, \text{with}\label{eq:offpolicy_weighting_expected} \!\\
m_t &\!=\! \lambda_t i_t + \gamma_t(1\!-\!\lambda_t)
f(s), \text{and }
f(s) \!=\! \E[F_{t}|S_t=s]\! \,,\!\nonumber
\end{align}
and using Eq.~\eqref{eq:offpolicy_weighting_expected} in Eq.~\eqref{eq:weighted_eligibility_trace}.
The expected follow-on can be estimated with a function $f_\vphi \approx f$, with learnable parameters $\vphi$. Cannonical learning methods for learning value functions can be applied, by reversing the direction of time in the learning update, similarly to expected eligibility traces\footnote{Estimating a single scalar, instead of a $d$-dimensional vector.}, e.g., Monte-Carlo regression on the instantaneous follow-on trace,
or backward TD (see \cref{A:Weightings for distribution correction: off-policy expected emphasis}). Because of its similarity to the $n$-step algorithm X-ETD($n$) proposed by \citet{ray2021} we will refer to this algorithm as \textbf{X-ETD($\lambda$)}---with X($0$)-ETD($\lambda$) denoting the variant where the expected follow-on is learned with backward TD and X($1$)-ETD($\lambda$) the variant where it is learned with regression to the full Monte Carlo follow-on trace.

\begin{figure}[t!]
\begin{center} \includegraphics[scale=.3]{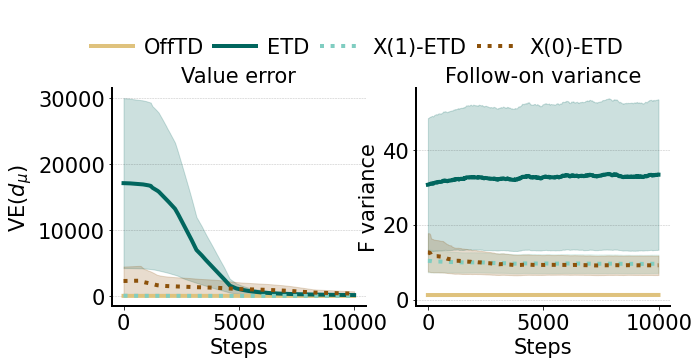}
    \end{center}
\caption{\label{fig:five_states_experiment}
\small \textbf{Five-state MRP:} \textbf{(Left)}: Value error under the stationary distribution of the behaviour policy for {\color{yellow}off-policy TD($0$)}, {\color{dark turquoise}ETD($0$)}, {\color{turquoise}X($1$)-ETD($0$)}---which learns emphasis by regressing on the instantaneous trace, and {\color{brown}X($0$)-ETD($0$)}---which learns emphasis by backward TD. Off-policy TD and X(1)-ETD visually overlap on the left, as do X(0)-ETD and X(1)-ETD on the right. The higher error of ETD is due to higher variance, caused by the follow-on trace. \textbf{(Right)}: Variance of the follow-on (taken to be 0 for off-policy TD, which does not have this trace).  Details on the MRP are in Fig.~\ref{fig:5state_mdp_app}, \cref{A:Weightings for off-policy distribution correction} and in \citet{sutton2016emphatic}, but do not affect this conclusion: follow-on traces typically add variance to the updates.} 
\end{figure}

\paragraph{Empirical illustration on emphatic algorithms.}
We illustrate these policy-evaluation algorithms on the $5$-state MRP introduced by \citet{sutton2016emphatic} (also depicted in Fig.~\ref{fig:5state_mdp_app} in \cref{A:Weightings for off-policy distribution correction}). In Fig.~\ref{fig:five_states_experiment}, we observe that the emphatic algorithms using expected emphasis have lower variance (see \cref{A:Weightings for off-policy distribution correction} for more details on the experimental setup).

\subsection{Selectivity for on-policy TD}\label{subsection:Selectivity for on-policy TD}{We now describe selectivity functions which make on-policy learning stable. For constant interest $i$ and constant discount factor $\gamma$, we can find a closed-form weighting that exactly equals the expected emphasis, without needing to learn it, and thus corrects for a dynamic trace-decay $\lambda_t$, thereby avoiding divergence. This weighting is coupled with the trace-decay $\lambda$ and the discount factor $\gamma$ through:
\begin{align}
\omega_t \!=\! (1 \!- \!\gamma\lambda_t)/(1\!-\!\gamma)
\!\iff\! \lambda_t \!=\! (\gamma \omega_t \!+\! (1\!-\! \omega_t))\!/\!\gamma\,.\!
\label{eq:coupling_1}
\end{align}
The constant denominator determined by $\gamma$ can be folded into the learning rate.

An interesting consequence of this new insight is that we could pick the gradient weighting in the accumulating trace \eqref{eq:weighted_eligibility_trace} to guarantee convergence, or, if such a weighting is given, we can use this to pick the trace-decay parameter $\lambda_t$ and the temporal discounting $\gamma$ to ensure convergence. 

If the discount $\gamma_t$ is dynamic, the expected emphasis cannot be recovered in closed-form. We can instead use a slightly different coupling:
\begin{align}
\!\omega_t \!=\! (1 \!-\! \gamma_t \lambda_t)\!/\!(1 \!-\! \beta_\lambda) \!\iff\! \lambda_t  \!=\! (1\!-\!\omega_t\!+\!\beta_\lambda\omega_t)\!/\!\gamma_t\!
\label{eq:coupling_2}
\end{align}
which is stable and convergent $\forall \beta_\lambda \in [0,  1)$ under mild conditions. Here, $\beta_\lambda$ controlls the decay rate of the on-policy follow-on trace, analogous to the one introduced by \citet{hallak2016} (allowing smooth interpolation between TD and ETD). 
\cref{A:Weightings for on-policy learning: on-policy expected emphasis} contains proofs and derivations.

\paragraph{Q($\lambda, \omega$)}{
Analogous to TD($\lambda, \omega$), we use Q($\lambda, \omega$) to refer to the generic algorithm that adds an explicit weighting function $\omega$ to the trace of the action-value function: 
\begin{align}
    \tilde{\e}_t &= \gamma_t \lambda_t\tilde{\e}_{t-1} +\omega_t\nabla\Q_{\w_t}(S_t, A_t)\,,\label{eq:selective_q_trace}
\end{align}
to be used in Eq.~\eqref{eq:delta_q} in place of $\e_t$.
}

\paragraph{Empirical illustration in deep reinforcement learning}{
To illustrate the importance of this novel connection between weightings and bootstrapping, we used Ms.Pac-Man, a canonical Atari game. We designed the following experimental setup to test the selective Q($\lambda_t, \omega_t$) (with the \emph{``$t$''} subscript denoting state or time-dependence), using the trace-decay correction rules in Eq.~\eqref{eq:coupling_1} and Eq.~\eqref{eq:coupling_2}.
With probability $\epsilon$, the agent's observation is replaced with random Gaussian noise, to mimic a noisy observation sensor. To simulate access to a module that detects such noisy observations, we provide access to a time-dependent interest $i_t$, capturing whether an observation is noisy or not, s.t. $i_t = 0$ if the observation at time step $t$ is noisy, and $i_t=1$ otherwise. Selectivity is entirely dictated by interest, with no other corrections: $\omega_t \equiv i_t$. For the state-dependent trace-decay function $\lambda_t$, we use Eq.~\eqref{eq:coupling_2}. In Fig.~\ref{fig:ms_pacman}--Top-Left \& Top-Right we observed that coupling the weighting and the trace-decay function in Eq.~\eqref{eq:coupling_2} recovers the baseline's performance, displaying robustness to observation noise. The caption and \cref{A:Weightings for on-policy distribution correction} contain further details.
}

\begin{figure}[t!]
\begin{center}
        \includegraphics[scale=.32]{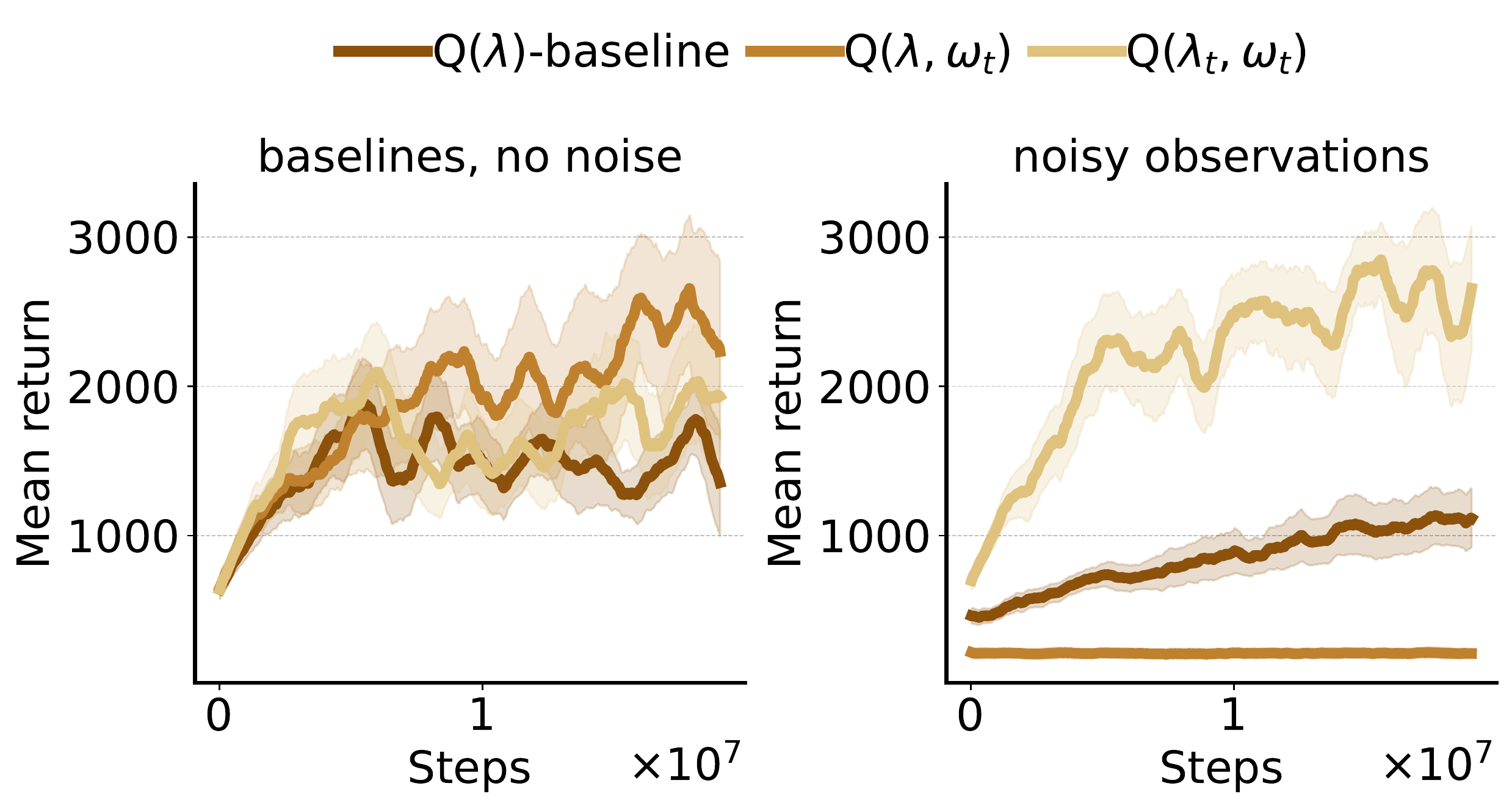}
        \includegraphics[scale=.32]{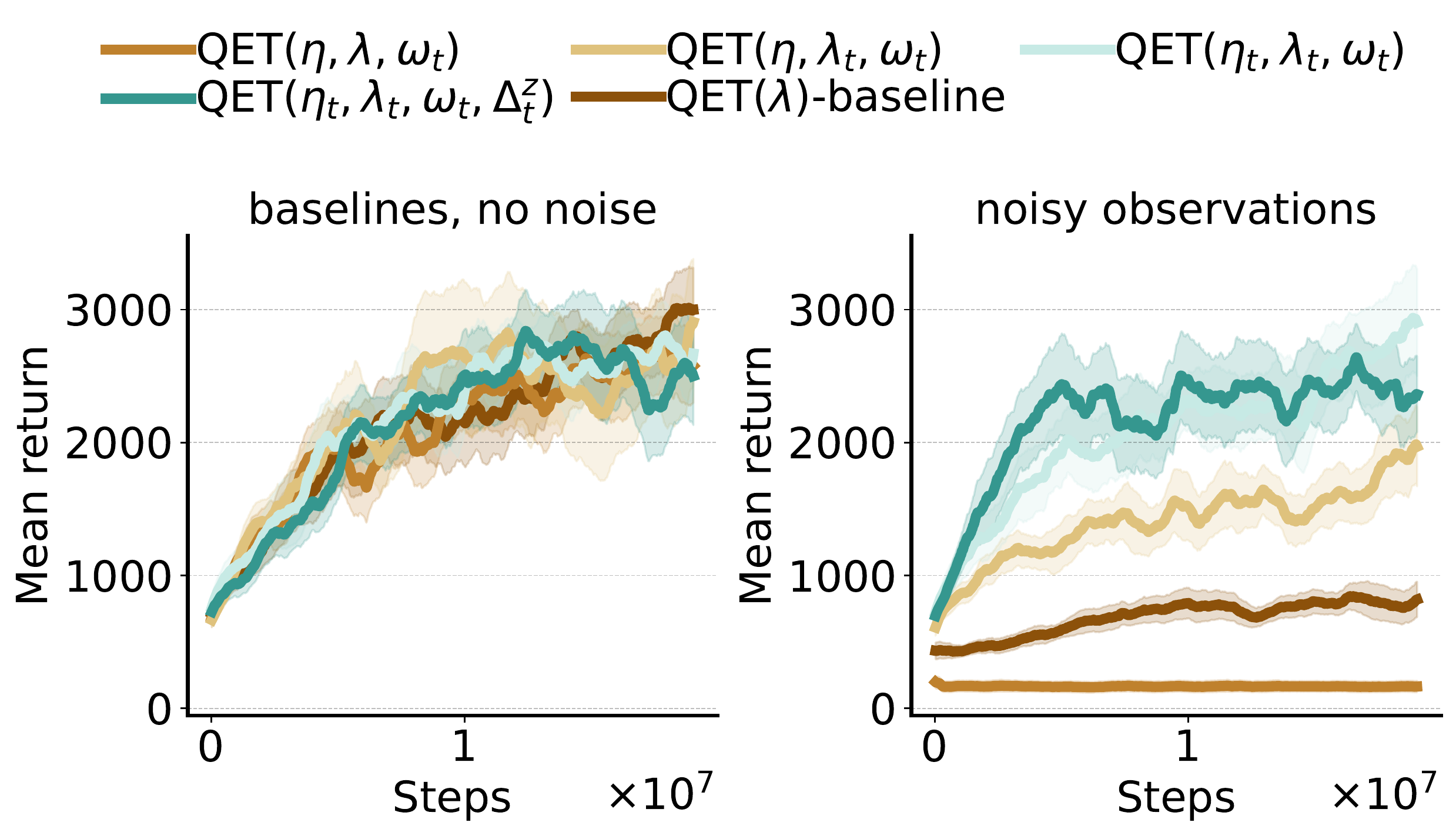}
    \end{center}
\caption{\label{fig:ms_pacman}
\small \textbf{Atari Ms.Pac-Man:}
Mean returns.
 \textbf{(Left)}: with standard observations, $\epsilon=0$. \textbf{(Right)}: noise is added with probability $\epsilon=0.5$. 
\textbf{(Top)} Algorithms using instantaneous traces: (i) {\color{brown}baseline Q($\lambda=0.9$)} using uniform weightings; (ii) {\color{light brown}Q($\lambda=0.9, \omega_t$)}, using state-dependent weightings: $\omega_t \equiv i_t$ (where $i_t = 0$ for noisy observations, and $1$ otherwise); (iii) {\color{yellow} Q($\lambda_t, \omega_t$)}, additionally using state-dependent trace-decays $\lambda_t$ with Eq.~\eqref{eq:coupling_2}, and $\beta_\lambda = 0.9$.
\textbf{(Bottom)} Algorithms using expected traces: (i) the {\color{brown} baseline QET($\eta=0, \lambda=0.9$)} using uniform weightings; (ii) {\color{light brown} QET($\eta=0, \lambda=0.9, \omega_t$)}, using $\omega_t = i_t$; (iii) {\color{yellow} QET($\eta=0, \lambda_t, \omega_t$)}, additionally using a state-dependent $\lambda_t$ with Eq.~\eqref{eq:coupling_2}, and $\beta_\lambda =0.9$; (iv) {\color{turquoise} QET($\eta_t, \lambda_t, \omega_t$)}, in addition applying a state-dependent trace-bootstrapping function $\eta_t$ using Eq.~\eqref{eq:coupling_3}, with $\beta_\eta=0$; (v) {\color{dark turquoise} QET($\eta_t, \lambda_t, \omega_t, \tilde{\Delta}^{z}_t$)}, also accounting for the weighting in the trace learning process, using Eq.~\eqref{eq:coupling_4}.
Shaded areas show standard error over $5$ seeds. These results demonstrate the importance of (1) weighting appropriately, and (2) coupling the weighting and $\lambda$ and $\eta$, as well as (3) scalability of the ideas to non-linear function approximation such as deep neural networks.
}
\end{figure}

\section{Planning selectively}
We now provide insights on explicitly adding selectivity to planning algorithms.
Expected eligibility traces attempt to capture all possible trajectories coalescing into a state. This allows credit to flow more broadly, not just to states that have happened, but also to states that could have happened, under the same policy. These methods can be interpreted as planning backwards, while their standard backward view counterparts---TD($\lambda$) and Q($\lambda$)---do not plan and only use the current sampled trajectory.

We consider adding explicit selectivity to these planning algorithms, and describe how different choices yield interesting new interpretations for these models of expected traces and their associated algorithms.

\subsection{Learning off-policy counterfactuals online}
\paragraph{ET($\lambda, \omega$)}{
For evaluation, we call ET($\lambda, \omega$), the planning algorithm that estimates and uses a model of the \emph{selective expected eligibility traces} in place of the instantaneous traces:
\eq{
\tilde{\z}_{\T}(s) &\approx \E_{\pi}\left[\tilde{\e}_{t} \mid S_{t}=s\right] \,,
}
with $\T$---the model's parameters, and $\tilde{\e}$ from Eq.~\eqref{eq:weighted_eligibility_trace}.
}

Expected traces have been described for the on-policy setting.
We can instead learn the expected trace for a target policy $\pi$ under a different behaviour $\mu$ by adding importance sampling ratios to the selective trace:
\eq{
\tilde{\e}_t &= \rho_t \gamma_t\lambda_t \tilde{\e}_{t-1} + \omega_t \nabla \V_{\w_t}(S_t)\,,
}
and choosing a selectivity function that adds \emph{emphatic corrections}, with $\omega_t$ as defined in equation~\eqref{eq:offpolicy_weighting_mixture_MC} or \eqref{eq:offpolicy_weighting_expected}.

\begin{figure}[t!]
\begin{center}
        \includegraphics[width=0.48\linewidth]{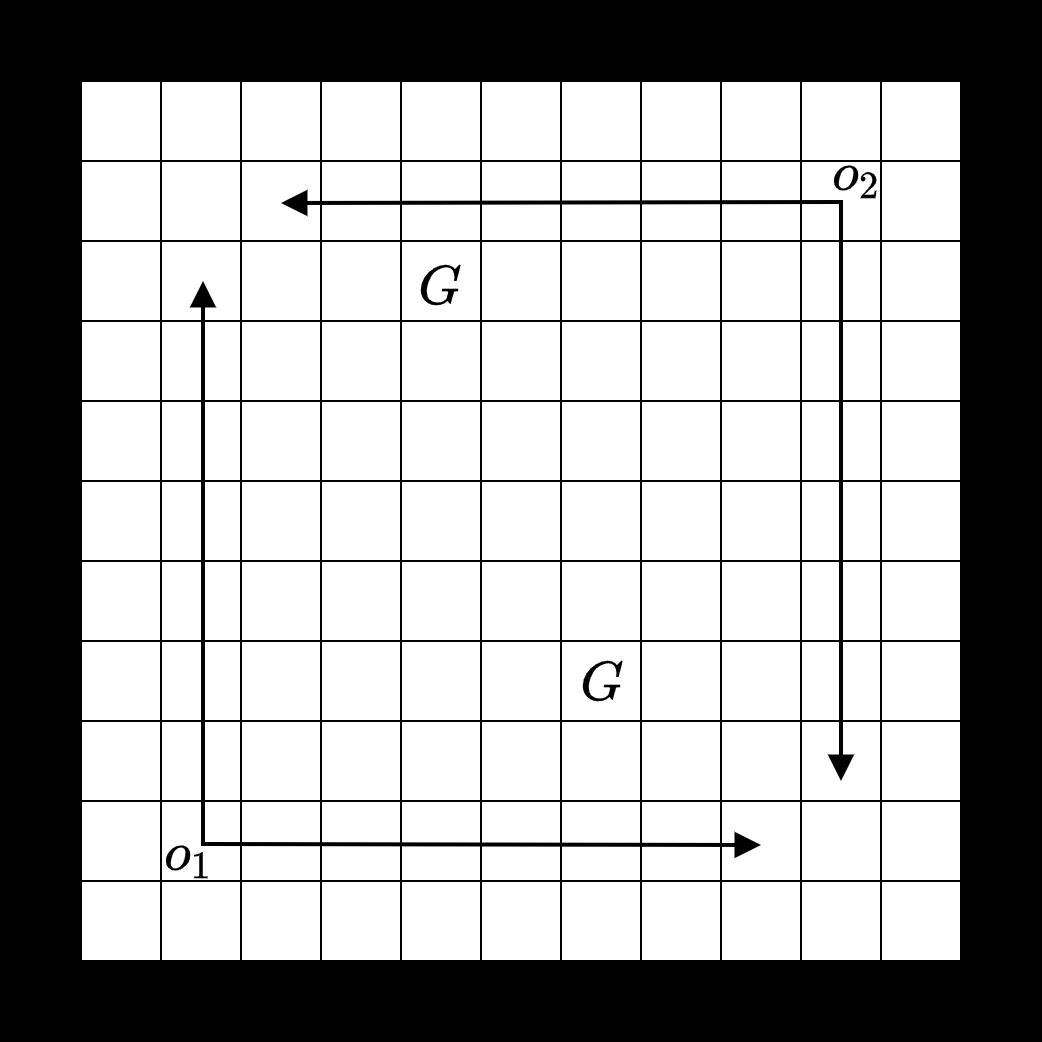}
        \includegraphics[width=0.48\linewidth]{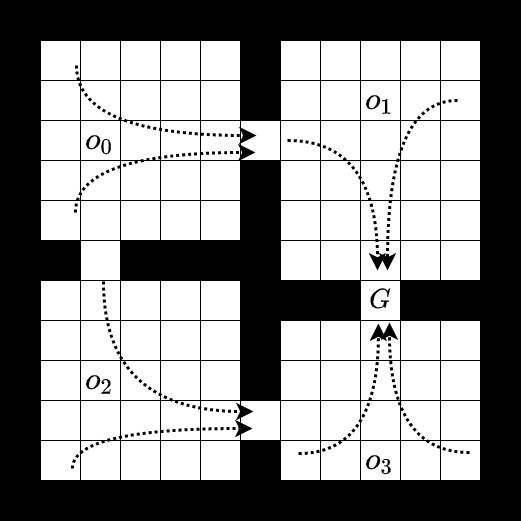}
    \end{center}
\caption{\label{fig:envs}
\small \textbf{(Left)} Open World. \textbf{(Right)} Four rooms domain.}
\end{figure}

\paragraph{QET($\lambda, \omega$)}{
We call selective QET($\lambda, \omega$) the generic algorithm using an expectation model of the selective action-value function trace from Eq.\eqref{eq:selective_q_trace}:
\begin{align}
    \tilde{\z}_{\T}(s) \approx \E_{\pi}[ \gamma_t \lambda_t \tilde{\e}_{t-1} \mid S_t = s]\label{eq:selective_qet}\,.
\end{align}
An important difference to prior algorithms is that we estimate the decayed \emph{previous} expected trace $\gamma_t \lambda_t \tilde{\e}_{t-1}$ in Eq.~\eqref{eq:selective_qet}, instead of $\tilde{\e}_t$. We then use $\tilde{\z}_{\T}(s) + \nabla \Q_{\w}(s, a)$ as the trace for action $a$ to be used in Eq.~\eqref{eq:delta_q} in place of $\tilde{\e}_t$. This avoids having to condition $\tilde{\z}_{\T}(s)$ on the action $a$, significantly reducing computation in settings with many actions---for instance in Atari this saves a factor $|\mathcal{A}|=18$, and in many domains $|\mathcal{A}|$ is higher. In addition we get more data per state than for each, more specific, state-action pair, thereby potentially facilitating learning the traces accurately
}

\paragraph{Empirical illustration of  learning counterfactuals (models of expected traces) online off-policy}{
Consider the Open World environment illustrated in Fig.~\ref{fig:envs}-Left. The precise setup is described in \cref{A:Off-policy counterfactual evaluation}. In short, the agent's behaviour $\mu$ is uniformly random and we consider learning about two stochastic policies: one that tends up and right and another that tends down and left.  We consider learning two expectation models for the traces associated with those two policies, and then use those expected traces to learn to predict their values. Rewards are, noisily, obtained when bumping into one of the goals (denoted $G$).  Bumping into a goal ends the episode; new episodes start at a random location. In Fig.~\ref{fig:room}, we illustrate the effect of  increasing the sparsity of the reward signal, with $\epsilon_r$ indicating the probability that the agent receives a reward of $r = 10/\epsilon_r$ in each goal. The reward is zero otherwise. Fig.~\ref{fig:room} shows off-policy expected traces effectively reduce variance.
}

\begin{figure}[t!]
\begin{center}
        \includegraphics[scale=.24]{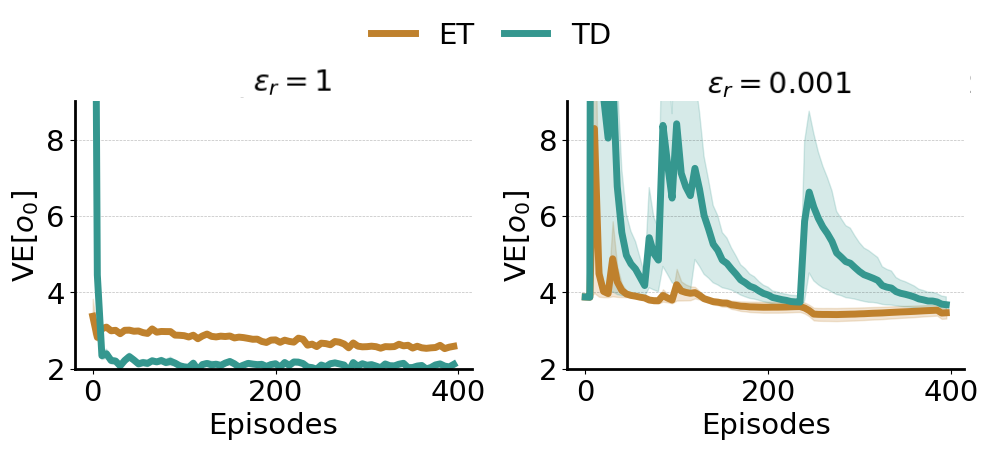}
    \end{center}
\caption{\textbf{Open World}: Value error for policy evaluation using {\color{dark turquoise}TD($\lambda=0.98$)} and {\color{brown}ET($\lambda=0.98$)}, for the policy tending to the top right (similar performance for the other policy in Fig.\ref{fig:room_o2}, in \cref{A:Off-policy counterfactual evaluation}).
With deterministic rewards ($\epsilon_r = 1$, left) TD($\lambda, \omega$) learned faster than ET($\lambda, \omega$). With stochastic rewards ($\epsilon_r = 0.001$, right), the off-policy expected traces reduce variance and value error. Details on hyper-parameters and tuning are in \cref{A:Off-policy counterfactual evaluation}.
Shaded areas show standard errors over $20$ seeds.
\label{fig:room}
\small }
\end{figure}

\subsection{Selectively using and learning models of traces}\label{section:selective_expected_traces}
To learn expected traces, \citet{hasselt2020} propose a mechanism similar to value bootstrapping by introducing \emph{mixture traces} $\e^\eta$ (analogous to $\lambda$-returns, with $\eta$ the counterpart of $\lambda$):
\begin{align}
\e^{\eta}_t &=\! (1 \!-\! \eta) \z_{\T_t}(S_t) + \eta (\gamma_t \lambda_t \e^{\eta}_{t-1} + \nabla \V_{\w_t}(S_t) )\,.\! \label{eq:mixture_traces_1}
\end{align}
The trace-bootstrapping parameter $\eta$ allows us to smoothly interpolate between using expected or instantaneous traces. Using $\e^{\eta=0}_t\!=\! \z_{\T_t}(S_t)$ results in counterfactual credit assignment based on expected traces, while using $\e^{\eta=1}_t \!=\!  \gamma_t \lambda_t \e^{\eta=1}_{t-1} + \nabla \V_{\w_t}(S_t)$ results in \emph{trajectory-based} learning, relying fully on instantaneous traces. 
The generic expected trace algorithm ET($\lambda$, $\eta$) is then defined by
$\Delta^\V_{t} = \delta_t \e^{\eta}_t$,
and smoothly interpolates between these extremes for $\eta \in [0, 1]$.


We now consider how selectivity could influence using expected traces for value learning. It is reasonable to rely on expected traces more in states where they are more accurate, and more on the instantaneous traces otherwise. Assuming an explicit selectivity mechanism, we can constrain the trace-bootstrapping parameter $\eta$:
\begin{align}
\eta_t = \beta_\eta \tilde{\omega}_t + (1 - \tilde{\omega}_t)\,, \label{eq:coupling_3}
\end{align}
where $\beta_\eta \in [0,1)$ allows for partial trace-bootstrapping, and we used $\tilde{\omega}$ to distinguish this from the $\omega$ used in the value learning process, which could be different. For instance, consider the special case in which selectivity captures partial observability. The value learning could then rely on estimated expected traces more for states that are less aliased, where $\tilde{\omega}_t > 0$, and on the instantaneous eligibility trace otherwise, where $\tilde{\omega}_t = 0$. 

So far, we discussed using selectivity for learning the value function, either through selective value updates or by adapting the mixing parameter $\eta$ of the mixture trace. But the learning of the expected trace is itself subject to a sampling procedure, in which we can inject selectivity. We can apply the same procedure we did for value learning in Eq.~\eqref{eq:weighted_TD_forward_update}:
\begin{align}
\tilde{\Delta}^{\z}_t &\equiv \tilde{\omega}_t \frac{\partial \tilde{\z}_{\T_t}(S_{t})}{\partial \T_t}\left(\tilde{\e}_t-{\z}_{\T_t}(S_{t})\right) \,,\label{eq:coupling_4}
\end{align}
and then using $\tilde{\Delta}^{\z}$ in place of $\Delta^{\z}$ in Eq.~\eqref{eq:learning_traces}. 

When the value learning uses mixture traces, the trace learning process can generally save function approximation resources by also focusing learning the expected trace only in those states in which the expected trace is used. If a mixture trace $\tilde{\e}^\eta$ (similar to Eq.~\eqref{eq:mixture_traces_1}) is used in Eq.~\eqref{eq:coupling_4} in place of $\tilde{\e}$, we now have a multi-step trace learning process similar to the value learning process in TD($\lambda$), so coupling the dynamic trace-bootstrapping $\eta$ and the weighting $\tilde{\omega}$ analogously ensures stable learning of the expected traces.

\paragraph{Empirical illustration in deep reinforcement learning.}{ 
We again consider Ms.Pac-Man to illustrate the effectiveness of using \emph{selective expected eligibility traces}, and the importance of coupling the trace-bootstrapping function $\eta$ with the weightings $\tilde{\omega}$, and focusing function approximation resources when learning the model for the trace. We use the same experimental setup as before. Fig.~\ref{fig:ms_pacman}--Bottom-Left illustrates the baseline runs of the algorithms without observational noise, whereas  Fig.~\ref{fig:ms_pacman}--Bottom-Right shows the effect of adding noise to the observations for all the algorithms. We found the algorithms using Eq.\eqref{eq:coupling_3} and Eq.\eqref{eq:coupling_4} to recover the original performance of the baselines, despite needing to rely on substantially noisier observations.
}

\subsection{Sparse expected eligibility traces}
Interestingly, learning expectation models of selective traces (c.f. \eqref{eq:weighted_eligibility_trace}) with \emph{binary} weighting functions $\omega : \S \to \{0, 1\}$, results in \emph{sparse expected eligibility traces}.
These models are equivalent to expected temporally-extended backward models \citep{chelu2020forethought} (proof in \cref{A:Sparse expected eligibility traces}).

In general, backward planning \citep{ Peng1993EfficientLA, Moore2004PrioritizedSR, McMahan2005FastEP, Sutton2008DynaStylePW, Hasselt2019WhenTU, chelu2020forethought} propagates credit to events \emph{possibly} responsible for the current outcome, typically using \emph{explicit  backward transition models}. Option models \cite{Sutton1999BetweenMA} describe the long-term effects of temporally-abstract actions (options).

Sparse expected eligibility traces, used in planning algorithms, e.g., selective ET or QET, assign credit in ways akin to jumpy backward planning \citep[cf.][]{chelu2020forethought}, but without learning an explicit dynamics model. Specifically, for \emph{binary} weighting functions, the selective planning algorithm QET($\lambda, \omega$) is similar to planning with \emph{backward option models} \citep{chelu2020forethought}, skipping parts of the state space as dictated by the weighting function. 

Credit assignment using sparse expected traces then happens in a sub-MDP within the original one, where the state space now contains only the states captured by the binary weighting function, and the action space is given by an induced option space (containing all the actions/sub-policies in-between these states). In terms of learning, this leads to a way of doing temporally-extended credit assignment in the original MDP. One can even go a step further, and apply selective temporal-discounting, similarly to \cite{pmlr-v97-harutyunyan19a}, e.g., using different discounting parameters for the intra-option expected eligibility trace that flows credit inside an option, and the sparse expected trace that flows credit over longer time-spans using options.

\begin{figure}[t!]
\begin{center}
        \includegraphics[width=0.9\linewidth]{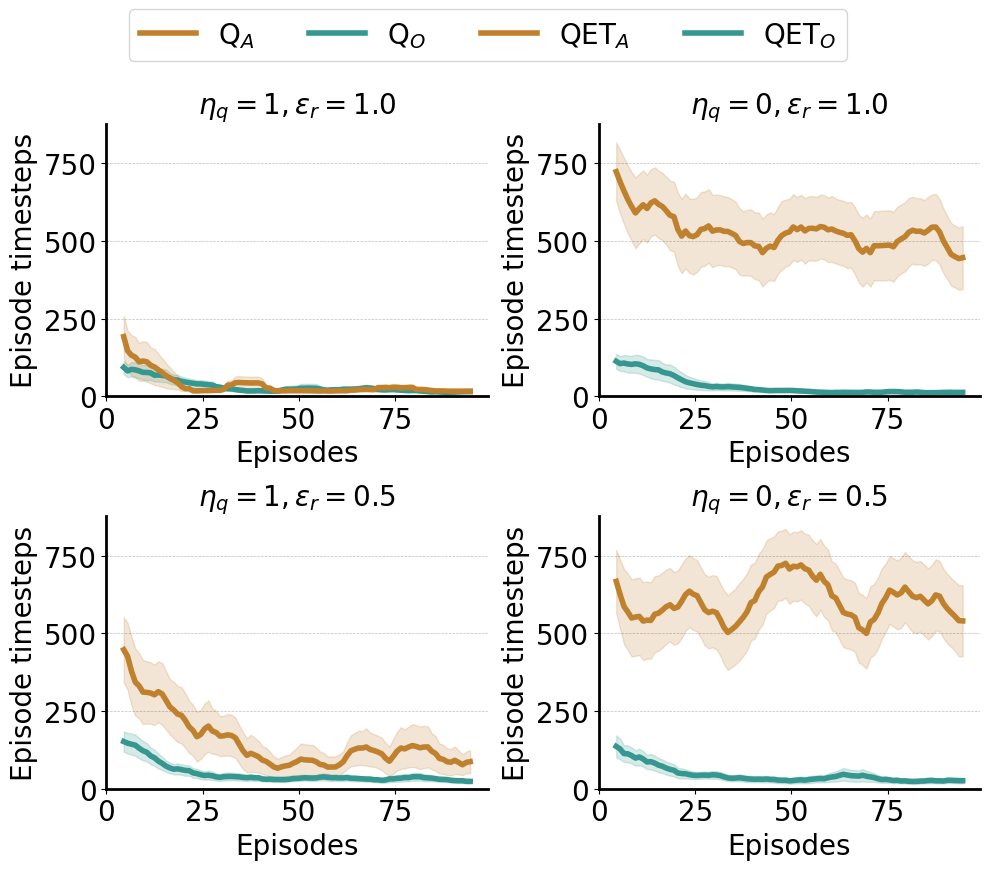}
    \end{center}
\caption{\label{fig:4rooms_illustration}
\small Number of steps per episode. \textbf{(Left)} {\color{light brown}Q$_A$}: {\color{light brown}Q($\lambda=0.9$)} with dense (standard) eligibility traces over primitive actions; {\color{dark turquoise}Q$_O$}:  {\color{dark turquoise}Q($\lambda_t, \omega_t$)} using a sparse eligibility trace for a policy over options (with pre-learned options as illustrated in Fig.~\ref{fig:envs}-Right, $\lambda_t$ using Eq.\eqref{eq:coupling_2}, $\beta_\lambda=0.9$ and $\omega_t = 1$ in hallways, and $0$ otherwise). \textbf{(Right)} {\color{light brown}QET$_A$}: {\color{light brown}QET($\eta=0, \lambda=0.9$)} using dense expected eligibility traces over primitive actions; {\color{dark turquoise}QET$_O$}: {\color{dark turquoise}QET($\eta_t, \lambda_t, \omega_t$)} using sparse expected eligibility traces for a policy over pre-learned options (with $\eta_t$ using Eq.\eqref{eq:coupling_3}, $\beta_\eta=0$, $\lambda_t$ using Eq.\eqref{eq:coupling_2}, $\beta_\lambda=0.9$, and $\omega_t=1$ only in hallways, and $0$ otherwise). \textbf{(Top)}: $r=10$ at the goal. \textbf{(Bottom)}: sparse reward signals, $r=20$, with probability $\epsilon_r=0.5$, and $r=0$ otherwise. Sparse expected traces were learned faster and coped better with sparse rewards.}
\end{figure}

\paragraph{Empirical evaluation of sparse expected traces}{
In the Four Rooms domain \cite{Sutton1999BetweenMA} (Fig.~\ref{fig:envs}-Right),
the agent aims to navigate to a goal location via options that take it from inside each room to its hallways, as shown in the illustration. The reward is $0$ everywhere, except at the goal, where it is $10 / \epsilon_r$ with probability $\epsilon_r$, and $0$ otherwise. The option policies are pre-learned, and illustrated in Fig.~\ref{fig:envs}-Right. Fig.~\ref{fig:4rooms_illustration} shows the results of using sparse traces for a policy over options, compared with using only primitive actions (see \cref{A:Empirical Sparse expected eligibility traces} for more details).
}

\section{Discussion}\label{section:discussion}
We discussed the use of weighting functions to stabilize, and focus resources in value-based credit assignment, introducing \emph{selective traces}. We illustrated the importance of trace corrections even for the on-policy case, as well as the significance of adding corrections for off-policy learning.
In the context of expected traces, weightings can act as a guide for when to rely on the learned expected traces. It can also be applied to selectively learn the traces themselves. 

We identified potential issues with naively combining weightings with instantaneous and expected traces. Using these as motivation, we proposed and investigated different modifications that allow for safe, selective credit assignment, identifying sufficient conditions linking the credit assignment parameters to ensure stable learning.

Our computational examples illustrate the potential benefits of adding weightings to the credit assignment problem, and show that selective learning can substantially improve performance in some settings.

Although specific weightings have been discussed for off-policy learning before, our definition is generic and does not restrict to a specific weighting, or to off-policy evaluation. Specifically, we consider weightings to focus function approximation resources, and for sparse or jumpy backward planning in the control setting. Emphatic weightings are one instance of selective traces. Other, perhaps more effective, choices are possible. Finding the ``best'' selection function for specific problems remains an intriguing open problem for future work, and could be domain-specific.

\citet{anand2021preferential} proposed a trace correction similar to the closed-form expected emphasis correction we derived, without recognizing that this is related to expected emphasis, and employing it differently. Particularly, they replace the trace-decay $\lambda$ with a preference function $\beta$, and allow zero-step returns rather than the typical $n$-step return, $n \geq 1$. Our formulation is more generic, subsuming TD, ETD, and X-ETD (expected-emphasis ETD) (\cref{Selectivity through emphasis for stability}), and recognizing the coupling between the weighting $\omega$ and the trace-decay $\lambda$ is strongly related to expected emphasis.
}

\textbf{Future work}~~Our examples consider the weighting given, or use emphatic weightings to stabilize learning. One direction for future work is adapting the weighting over time, based on experience. One potential approach is (meta) learning $\omega$, $\lambda$, and/or $\gamma$ \cite{Sutton1994OnSA, White2016AGA, Xu2018MetaGradientRL,zahavy2020self}, based on variance or bias minimization, or other proxy objectives \cite{Kumar2020DisCorCF}. Our results can be used to derive stable convergent updates, by respecting the coupling of these parameters.

A second potential approach is to use ideas for learning option terminations \citep{Bacon2017TheOA, Harutyunyan2019TheTC}, which are akin to the credit assignment functions we considered. Inferring controllability \cite{Harutyunyan2019HindsightCA} could inform weightings over states and actions, e.g., some states are irrelevant if all actions have the same consequences, and some actions could be irrelevant if they are unlikely to be selected. Weightings can be learned based on these intuitions. Hindsight conditioning \cite{Harutyunyan2019HindsightCA} can also help infer
policy-related weightings.

\section*{Acknowledgements and disclosure of funding}
Veronica Chelu was partially supported by a Borelis AI fellowship, and the paper partially executed while an intern at DeepMind.
\bibliographystyle{abbrvnat}
\bibliography{main}

\newpage
\appendix
\onecolumn
\input{appendix.tex}

\end{document}

%% file: math_commands.tex
\newcommand{\eq}[1]{\begin{align*}#1\end{align*}}
\newcommand{\eqq}[1]{\begin{align}#1\end{align}}
\def\w{\mathbf{w}}

\def\x{\mathbf{x}}
\def\E{\mathbb{E}}
\def\R{\mathbb{R}}

\def\V{\mathbb{V}}

\def\X{\mathbf{X}}
\def\A{\mathbf{A}}
\def\D{\mathbf{D}}
\def\d{\mathbf{d}}
\def\b{\mathbf{b}}

\def\P{\mathbf{P}}
\def\I{\mathbf{I}}

\def\R{\mathbb{R}}
\def\X{\mathbf{X}}

\def\r{\mathbf{r}}

\def\bGamma{\mathbf{\Gamma}}

\usepackage{subfigure,wrapfig}

\def\e{\mathbf{e}}
\def\X{\mathbf{X}}

\def\bGamma{\mathbf{\Gamma}}
\def\T{\mathbf{\Theta}}

\def\P{\mathbf{P}}
\def\K{\mathbf{K}}
\def\D{\mathbf{D}}

\def\e{\mathbf{e}}

\def\X{\mathbf{X}}
\def\T{\mathbf{\Theta}}

\def\P{\mathbf{P}}
\def\D{\mathbf{D}}

\def\S{\mathscr{S}}
\def\calS{\mathscr{S}}
\def\calA{\mathscr{A}}
\def\R{\mathbb{R}}
\def\calP{\mathscr{P}}

\def\bLambda{\mathbf{\Lambda}}

\def\calH{\mathscr{H}}
\def\z{\mathbf{z}}

\def\x{\mathbf{x}}
\def\e{\mathbf{e}}

\def\w{\mathbf{w}}

\def\D{\mathbf{D}}
\def\P{\mathbf{P}}
\def\X{\mathbf{X}}

\def\I{\mathbf{I}}
\def\A{\mathbf{A}}
\def\b{\mathbf{b}}
\def\d{\mathbf{d}}

\def\r{\mathbf{r}}

\def\z{\mathbf{z}}
\def\x{\mathbf{x}}
\def\e{\mathbf{e}}

\def\vphi{\mathbf{\varphi}}

\def\w{\mathbf{w}}

\def\D{\mathbf{D}}
\def\P{\mathbf{P}}
\def\X{\mathbf{X}}

\def\I{\mathbf{I}}
\def\A{\mathbf{A}}
\def\b{\mathbf{b}}
\def\d{\mathbf{d}}

\def\r{\mathbf{r}}

\def\TD{\operatorname{TD}}

\def\Q{{\operatorname{Q}}}
\def\V{{\operatorname{V}}}

\usepackage[mathscr]{eucal}

%% file: appendix.tex
\title{Supplementary Material}
\section{Background and preliminaries (details)}
\subsection{Expected eligibility traces}\label{A:Expected eligibility traces}
\emph{Expected eligibility trace} (ET) algorithms \citep{hasselt2020} have been introduced for \emph{off-trajectory}, \emph{on-policy} value learning, replacing the instantaneous trace ${\e}_t$ with an estimated expectation:
\begin{align}
\z_{\T}(s) & \approx \E_{\pi}\left[{\e}_{t} \mid S_{t}=s\right]\,,\text{ with}\nonumber\\
\e_t &= \gamma_t \lambda_t \e^{\eta}_{t-1} + \nabla \V_{\w_t}(S_t)\,.\nonumber
\end{align}
We can approximate the expected traces by regressing on the instantaneous eligibility trace $\e_t$, or by a mechanism similar to value bootstrapping using the \emph{mixture trace} $\e^\eta_t$:
\begin{align}
\e^{\eta}_t & = (1 - \eta) \z_{\T_t}(S_t) + \eta [\gamma_t \lambda_t \e^{\eta}_{t-1} + \nabla \V_{\w_t}(S_t) ]\nonumber\,.
\end{align}
The trace-bootstrapping parameter $\eta$ specifies the credit assignment mechanism with $\e^{\eta=0}_t= \z_{\T_t}(S_t)$ resulting in \emph{counterfactual} credit assignment (fully relying on the expected traces), while $\e^{\eta=1}_t =  \gamma_t \lambda_t \e^{\eta=1}_{t-1} + \nabla \V_{\w_t}(S_t)$ uses \emph{factual} or \emph{trajectory-based} learning (reverting fully to instantaneous traces). 

The expected trace algorithm ET($\lambda$, $\eta$) is then defined by:
\begin{align}
\w_{t+1} &= \w_{t} + \alpha^{\w_t}_{t}\Delta^{\V}_t\,, \text{ with}
\nonumber\\
\Delta^{\V}_t &= \delta_t \e^{\eta}_t \,,\nonumber
\end{align}
and smoothly interpolates between these two end points (using expected or instantaneous traces).

The expected trace $\z_\T(S_t)$ can be learned by regressing on $\gamma_t \lambda_t \e^{\eta}_{t-1} + \nabla\V_{\w_t}(S_t)$, s.t.:
\begin{align}
    \Delta^{\z}_t &\equiv \frac{\partial {\z}_{\T_t}(S_{t})}{\partial \T_t}\Big(\gamma_t \lambda_t \e^{\eta}_{t-1} + \nabla\V_{\w_t}(S_t) -{\z}_{\T_t}(S_{t})\Big)\,,\nonumber\\
\T_{t+1} &= \T_{t} + \alpha^{\T}_t \Delta^{\z}_t  \textbf{}\,,
\end{align}
or using a different mixture parameter $\tilde{\eta}$ for learning the trace:
\begin{align}
\e^{\tilde{\eta}}_t = (1 - \tilde{\eta}) \z_{\T_t}(S_t) + \tilde{\eta}[\gamma_t\lambda_t \e^{\tilde{\eta}}_{t-1} + \nabla \V_{\w_t}(S_t)]\,,\nonumber
\end{align} 
 keeping $\eta$ just for determining when to use the model. So updating $\T$ with
\eq{
\Delta^{\z}_t &\equiv \frac{\partial {\z}_{\T_t}(S_{t})}{\partial \T_t}\Big(\gamma_t\lambda_t\e^{\tilde{\eta}}_{t-1}  + \nabla \V_{\w_t}(S_t)-{\z}_{\T_t}(S_{t})\Big)\,,\nonumber\\
\T_{t+1} &= \T_{t} + \alpha^{\T}_t \Delta^{\z}_t \nonumber \,,
}
where $\alpha^{\T}_t\in (0,1)$ is a (possibly time-varying) step-size parameter and $\tilde{\eta}$ is analogous to $\lambda$, but for learning traces rather than values, using TD methods when $\tilde{\eta} = 0$, and Monte-Carlo regression when $\tilde{\eta}=1$. Note that, like TD(0), we always include at least one sampled transition, even if $\tilde{\eta} = 0$.

\section{Proofs and derivations}
\subsection{New derivation of off-policy Q($\lambda$)}\label{A:New derivation of Off-policy}
Let $\rho(s, a) = \frac{\pi(a|s)}{\mu(a|s)}$ be the importance sampling ratio between the target policy $\pi$ and the behaviour $\mu$.
In the trajectory context, we may write for any $k$ and $t$: $\rho_k^t = \prod_{j=k}^{t}\rho_j, \text{ with } \rho_j = \rho(S_j, A_j)$,
with the convention that $\rho_k^t = 1$ for $t<k$.

With this notation, the off-policy importance sampled $\lambda$-return for state-dependent $\lambda$ and $\gamma$ can be written as:
\begin{align}
G_t^{\lambda,\pi} &= (1 - \lambda_{t+1}) G^{\pi}_{t:t+1} + \lambda_{t+1}(1-\lambda_{t+2})G^{\pi}_{t:t+2} + \dots \nonumber\\
&= (1-\lambda_{t+1})\rho_{t}[R_{t+1} + \gamma_{t+1} \V(S_{t+1})] + (1-\lambda_{t+2})\lambda_{t+1}\rho_{t}[R_{t+1} + \gamma_{t+1}\rho_{t+1} R_{t+2} + \gamma_{t+1}\rho_{t+1}\gamma_{t+2} \V(S_{t+1}) + \nonumber\\
&= \rho_{t}\left(R_{t+1} + (1-\lambda_{t+1})\gamma_{t+1}\V(S_{t+1})\right) + \lambda_{t+1}\gamma_{t+1}\rho_{t}\rho_{t+1}[R_{t+2} + (1-\lambda_{t+2})\gamma_{t+2}\V(S_{t+2})] + \dots \nonumber\\
&=\sum_{k=0}^{\infty}\left(\prod_{j=0}^{k-1} \gamma_{t+j+1}\lambda_{t+j+1}\rho_{t+j}\right)\rho_{t+k}\big[R_{t+k+1}+(1-\lambda_{t+k+1})\gamma_{t+k+1}\V(S_{t+k+1})\big]\nonumber
\end{align}
This is the return from some state $S_t$ onward. For the return from some state $S_t$ and action $A_t$, we can just drop the final importance sampling ratio, yielding:
\eq{
G_t^{\lambda,\pi} &=\sum_{k=0}^{\infty}\left(\prod_{j=0}^{k-1} \gamma_{t+j+1}\lambda_{t+j+1}\rho_{t+j}\right)\big[R_{t+k+1}+(1-\lambda_{t+k+1})\gamma_{t+k+1}\V(S_{t+k+1})\big].
}
For Q-learning the bootstrapped target is: $\V(S_{t+k+1}) \equiv \max_a \Q(S_{t+k+1}, a)$, whereas for SARSA we have: $\V(S_{t+k+1}) \equiv  \E_\pi[\Q(S_{t+k+1}, A)]$.

Let 
\eq{
\Delta^{\Q}_t &= (G^{\lambda,\pi}_t - \Q_\w(S_t, A_t)) \nabla \Q_\w(S_t, A_t) \\
&= G^{\lambda,\pi}_t  \nabla \Q_\w(S_t, A_t) - \Q_\w(S_t, A_t) \nabla \Q_\w(S_t, A_t)\,,
}
be the update to the parameters $\w$ at time step $t$.

We look at the first part of the update, which can be rewritten as:
\begin{align}
\sum_{t=0}^{\infty} G^{\lambda,\pi}_t  \nabla \Q_\w(S_t, A_t) &= \sum_{t=0}^{\infty} \sum_{k=0}^{\infty} \left(\prod_{j=0}^{k-1} \gamma_{t+j+1} \lambda_{t+j+1}\rho_{t+j}\right) \big[R_{t+k+1} + (1-\lambda_{t+k+1})\gamma_{t+k+1}\V(S_{t+k+1})\big]\nabla \Q_\w(S_t, A_t) \nonumber\\
&=  \sum_{t=0}^{\infty} \sum_{k=t}^{\infty} \left(\prod_{j=t}^{k-1} \gamma_{j+1} \lambda_{j+1}\rho_{j}\right) \big[R_{k+1} + (1-\lambda_{k+1})\gamma_{k+1}\V(S_{k+1})\big]\nabla \Q_\w(S_t, A_t) \nonumber\\
&=  \sum_{k=0}^{\infty} \sum_{t=0}^{k} \left(\prod_{j=t}^{k-1} \gamma_{j+1} \lambda_{j+1} \rho_{j}\right) \big[R_{k+1} + (1-\lambda_{k+1}) \gamma_{k+1}\V(S_{k+1})\big]\nabla \Q_\w(S_t, A_t)\nonumber \\
&=  \sum_{t=0}^{\infty} \sum_{k=0}^{t} \left(\prod_{j=k}^{t-1} \gamma_{j+1} \lambda_{j+1}\rho_{j}\right)\big[R_{t+1} + \gamma_{t+1}(1-\lambda_{t+1})\V(S_{t+1})\big]\nabla \Q_\w(S_k, A_k) \nonumber\\
&=  \sum_{t=0}^{\infty} \big[R_{t+1} + \gamma_{t+1}(1-\lambda_{t+1})\V(S_{t+1})\big]\underbrace{\sum_{k=0}^{t} \left(\prod_{j=k}^{t-1} \gamma_{j+1} \lambda_{j+1}\rho_{j}\right)\nabla \Q_\w(S_k, A_k)}_{\e_t},\nonumber
\end{align}
with
\begin{align}
\e_t &= \sum_{k=0}^{t} \left(\prod_{j=k}^{t-1} \gamma_{j+1} \lambda_{j+1}\rho_{j}\right)\nabla \Q_\w(S_k, A_k) \nonumber\\
&=\nabla \Q_\w(S_t, A_t) + \gamma_t \lambda_t \rho_{t-1} \nabla \Q_\w(S_{t-1}, A_{t-1}) + \dots\nonumber\\
&= \gamma_t\lambda_t\rho_{t-1} \e_{t-1} + \rho_t \nabla \Q_\w(S_t, A_t)\nonumber
\end{align}
Then, the full update is:
\begin{align}
\sum_{t=0}^{\infty} \Delta^Q_t &= \sum_{t=0}^{\infty} \underbrace{\left[R_{t+1} + \gamma_{t+1}(1-\lambda_{t+1}) \V(S_{t+1})\right]}_{R^{\lambda}_t}\e_t - \Q_\w(S_t, A_t) \nabla \Q_\w(S_t, A_t)\,,\nonumber
\end{align}
which is what we use in the new versions of Q($\lambda$) and QET($\lambda, \eta$) algorithms.

When using this definition with the QET algorithm,
one important difference to the prior algorithm introduced by \citet{van2020expected} is that we estimate the decayed \emph{previous} expected trace, so approximating $\z_{\T}(s) \approx \E_{\pi}[ \gamma_t \lambda_t \e_{t-1} \mid S_t = s]$, and then using $\z_{\T}(s) + \nabla \Q_{\w}(s, a)$ as the trace for action $a$.  This avoids having to condition the expected trace on the action, which can significantly reduce computation in some implementations.

\subsection{Sparse expected eligibility traces}
\label{A:Sparse expected eligibility traces}
\begin{proposition}
Sparse accumulating traces with binary state weighting functions are equivalent to temporally extended backward models.
\end{proposition}
\begin{proof}
From the definition of the selective expected trace, we have
\eq{
\tilde{\z}(s) &= \E_\pi[\tilde{\e}_t|S_t=s]\\
&= \E[\gamma_t\lambda_t \tilde{\e}_{t-1} + \nabla \V_{\w_t}(S_t) \omega_t |S_t = s]\\
&= \nabla \V_{\w_t}(s) \omega(s) + \gamma(s) \lambda(s) \sum_{\tilde{s}} P_\pi (s|\tilde{s})\tilde{\z}(\tilde{s}) 
}
If $\omega(s) = 1 - \gamma(s) \lambda(s)$ (cf. Eq.~\eqref{eq:coupling_2}), then
\begin{align}
\tilde{\z}(s) \!=\! \omega(s) \nabla \V_{\w_t}(s) \!+\! (1\!-\!\omega(s))\sum_{\tilde{s}} P_\pi (s|\tilde{s})\tilde{\z}(\tilde{s})  \label{eq_A:backward_model}
\end{align}
In Eq.~\eqref{eq_A:backward_model}, considering $c(s) \equiv \nabla \V_{\w_t}(s)$ a multi-dimensional cumulant and $\beta(s)\equiv 1-\omega(s)$ the probability of termination at state $s$, then $\tilde{\z}(s)$ can be interpreted as a temporally-extended option model for an option defined as: $o = (\pi, \beta)$, where $\pi$ is the option's policy and $\beta$ is the corresponding binary termination function.

\end{proof}

\subsection{Weightings for distribution correction: off-policy expected emphasis}\label{A:Weightings for distribution correction: off-policy expected emphasis}
We consider methods that learn the expectation of the follow-on weighting $f(s) = \E_\mu[F_t|S_t=s]$, resulting in a state-dependent selectivity function to be used in place of the history-dependent weighting $\omega_t$.
We call the emphatic algorithm, analogous to $\lambda$-discounted TD, resulting from the aforementioned approach, X-ETD($\lambda$) (where ``X'' comes from ``expected''):
\begin{align}
\omega_t &= \rho_t m_t \,, \text{with }\nonumber\\
m_t & = \lambda_t i_t + \gamma_t(1-\lambda_t)
f(s)\nonumber\,,
\end{align}
The expected follow-on can be estimated using a function $f_\vphi \approx f$, with learnable parameters $\vphi$. Learning methods, typically used to learn value functions, can be applied, by reversing the direction of time in the learning update, similarly to expected eligibility traces, e.g., Monte-Carlo regression on the instantaneous follow-on trace 
\eqq{
    \vphi_{t+1} = \vphi_t + \alpha^{\vphi}_t (F_t - f_{\vphi_t}(S_t))\nabla_{\vphi_t} f_{\vphi_t}(S_t) \label{eq_A:followon_regression}\,,
}
or backward TD:
\eqq{
    \vphi_{t+1} = \vphi_t + \alpha^{\vphi}_t (f_\vphi(S_{t-1}) + i_t - f_{\vphi_t}(S_t))\nabla_{\vphi_t} f_{\vphi_t}(S_t) \label{eq_A:followon_backwardtd}\,,
}
with $\alpha^{f}_t$ a (possibly time-varying) step-size.

Unfortunately, both methods can be problematic. The Monte-Carlo regression problem has targets with infinite variance, so it is not guaranteed to converge, whereas the backward TD method
suffers from ``off-policiness'', same as value learning. 

Similarly to the case of expected eligibility traces, an interpolation between the instantaneous follow-on trace and its estimated expectation is possible using a mixture trace:
\eq{
F^\eta_t & = (1 - \eta^{F}) f_\vphi(S_t) + \eta^{F} \gamma_t \rho_{t-1}F_{t-1} + i_t\,,
}
with $\eta^{F}$ the mixing parameter, interpolating between using the expected follow-on $f(s)$ or the history-dependent instantaneous follow-on trace $F_t$.
The mixture trace $F^\eta$ is then used in the selectivity function $\omega$ as:
\eq{
\omega_t = \rho_t(\underbrace{\lambda_t i_t + (1-\lambda_t) F^\eta_t}_{m_t}) \,,
}
with $m_t$ the mixed emphasis. We would then use $\omega$ to learn the value parameters $\w$.

A second mixture trace $e_t^{f}$ can be used as target in the estimation of the expected follow-on $f_\vphi$ (same as for expected traces), one that uses a different mixture parameter $\eta^f$ (possibly different from $\eta^F$, distinguishing the way the trace is learned from how it is used:
\begin{align}
e^{f}_t & = (1 - \eta^{f}) f_\vphi(S_t) + \eta^{f} \gamma_t \rho_{t-1}F_{t-1} + i_t\,, \label{eq_A:followon_mixture}
\end{align}

The expected mixture trace is then used as target for an estimated model s.t. $f_\vphi \approx \E[\gamma_t \rho_{t-1} e^f_{t-1}|S_t=s]$, with $\eta^{f}$ interpolating between Monte-Carlo regression on the instantaneous trace ($\eta^{f}=1$), and backward TD by bootstrapping on the expected trace ($\eta^{f}=0$). 
The follow-on trace learning process updates the trace parameters $\vphi$ with:
\eq{
\vphi_{t+1} & = \vphi_{t} + \alpha^{f}_t 
\left(\gamma_t \rho_{t-1}e^f_{t-1}-f_{\vphi_t}(S_{t})\right)\nabla_{\vphi_t} f_{\vphi_t}(S_t)\,,
}
with $\alpha^{f}_t$---the step size. We use X($\eta^{f}$)-ETD($\lambda$) to explicitly denote the mechanism used to learn the expected follow-on trace, with $\eta^{f} = 1$ for Monte-Carlo regression, and $\eta^{f} = 0$ for backward TD($0$).


\subsection{Weightings for on-policy learning: on-policy expected emphasis}\label{A:Weightings for on-policy learning: on-policy expected emphasis}
Emphatic TD uses a state weighting function of the form:
\eq{
\omega_t &= M_t \rho_t\\
&= \rho_t\lambda_t i_t + \rho_t (1-\lambda_t)F_t, \text{ with } F_t = \gamma_{t} \rho_{t-1} F_{t-1} + i_t \\
&= \rho_t i_t + \rho_t\gamma_{t} (1 - \lambda_{t})\sum_{k=1}^{t}\left(\prod_{j=k}^{t}\rho_{t-j}\gamma_{t-j}\right) i_{t-k}\\
 &= \rho_t i_t + \rho_t \gamma_t (1-\lambda_t)\rho_{t-1} i_{t-1} +  \rho_t \gamma_t (1 -\lambda_t) \rho_{t-1}\gamma_{t-1} \rho_{t-2} i_{t-2} \dots,
}

For the on-policy learning, we have $\rho_t = 1, \forall t$. For on-policy learning and uniform interest $i_t = 1, \forall t$, then, the state weighting becomes:
\eq{
\omega_t &= M_t \\
&= \lambda_t + (1 - \lambda_t) F_t, \text{ with } F_t = \gamma_{t} F_{t-1} + 1 \\
&= 1 + \gamma_t (1 - \lambda_t) \sum_{k=1}^{t} \left(\prod_{j=k}^t \gamma_{t-j} \right)
}
The expected emphasis in this setting is:
\eq{
f(s) &= \lim_{t\to\infty} \E_\pi[F_t | S_t =s]\\
&= \E_\pi\left[1 + \sum_{k=1}^{t} \prod_{j=k}^t \gamma_{t-j}\mid S_t= s\right]
}
\paragraph{Constant $\gamma$}{
Furthermore, for constant $\gamma$ we have:
\eq{
f(s) &= 1 + \gamma + \gamma^2 + \dots = \frac{1}{1 - \gamma}
}
Replacing the expected emphasis in the definition of the weighting $\omega$:
\eq{
\omega_t &= \lambda_t + (1 - \lambda_t)/(1-\gamma)\\
&= (\lambda_t - \lambda_t \gamma + 1 - \lambda_t)/(1-\gamma) \\
&= (1 - \lambda_t \gamma)/(1-\gamma)  \\
\implies
\lambda_t &= (\gamma\omega_t + (1 - \omega_t))/\gamma
}
For constant $\gamma$, we can omit the denominator, since it would be just re-scaling the update by a constant factor which can be folded into the learning rate, yielding:
\eq{
\omega_t = 1 - \gamma\lambda_t
\implies \lambda_t = \gamma \omega_t + (1 - \omega_t)
}
}
\paragraph{Adaptive $\gamma$}{
Let $\P_\pi$ be the transition matrix induced by following policy $\pi$, with $[\P_\pi]_{s,s^\prime} = P(s^\prime|s,a) \pi(a|s)$, and $[\P_\pi^\top]_{:,s} = P(\cdot|s,a) \pi(a|s)$ the vector corrsponding to all entries of the succesor states of $s,a$. Let $\bGamma$, and $\bLambda$ be diagonal matrices, the former representing the discount matrix---with diagonal entries $\gamma(s)$, and the latter the trace-decay matrix---using $\lambda(s)$ on its diagonal. Let $\d_\pi$ denote the vector, with entries coresponding to the stationary distributions $d_\pi(s)$. Then, in matrix notation, the following hold: 
\eq{
\P_\pi^{\top} \d_\pi &= \sum_s d_\pi(s) [\P_\pi^{\top}]_{:, s} = \d_\pi \implies
\bGamma \P_\pi^{\top} \d_\pi = \bGamma \d_\pi \\
\P_\pi^{\top} \bGamma \d_\pi &= \sum_s \gamma(s) d_\pi(s) [\P_\pi^{\top}]_{:, s}\,.
}
Furthermore, we can write the stationary distribution under $\pi$ reweighted by the follow-on weighting, in matrix notation:
\eq{
\d_\pi^f &= (\I - \bGamma \P_\pi^{\top})^{-1} \d_\pi \\
&= \d_\pi + \bGamma \P_\pi^{\top} \d_{\pi} + ((\bGamma\P_{\pi})^2)^{\top} \d_{\pi} + \dots\\
&= \d_\pi + \bGamma \P_\pi^{\top} \d_{\pi} + \bGamma\P_{\pi}^{\top} \bGamma\P_{\pi}^{\top}\d_{\pi} + \dots
}
Assuming we can bound the discount factor $\gamma(s)$ with $\beta_\lambda$, $\forall s \in \S$, s.t. $\gamma(s) \leq \beta_\lambda$,  then:
\begin{align}
\P_\pi^{\top} \bGamma \d_\pi &\leq \beta_\lambda \sum_s d_\pi(s) [\P_\pi^{\top}]_{:, s} = \beta_\lambda \d_\pi \implies \bGamma\P_\pi^{\top} \bGamma \d_\pi < \beta_\lambda^2 \d_\pi \label{eq:beta_bound}
\end{align}
Using the assumption in Eq.~\ref{eq:beta_bound}, we obtain the expected follow-on trace:
\eqq{
\d_\pi^f
&= \d_\pi + \bGamma \d_{\pi} + \bGamma \P_{\pi}^{\top} \bGamma \d_\pi + \dots\nonumber \\
&= (1 + \beta_\lambda + \beta_\lambda^{2} + \dots)\d_\pi \nonumber\\
&= (1 - \beta_\lambda)^{-1}\d_\pi \label{eq_A:d_f_beta}
}
Let $\Omega$ be a weighting matrix, subject to the constraint:
\begin{align}
\Omega &= \bLambda  + (\I - \bLambda)(\I - \bGamma \P_\pi^{\top})^{-1}\nonumber \\
&= (\I - \bLambda \bGamma \P_\pi^{\top})(\I - \bGamma \P_\pi^{\top})^{-1} \label{eq_A:matrix_coupling}\\
\implies \bLambda &= 
\I  - \Omega(\I - \bGamma \P_\pi^{\top})(\bGamma \P_\pi^{\top})^{-1} \nonumber\\
&= (\Omega \bGamma \P_\pi^{\top} +  (\I - \Omega))(\bGamma \P_\pi^{\top})^{-1}\nonumber
\end{align}
Inserting the result from Eq.~\eqref{eq_A:d_f_beta} in Eq.~\eqref{eq_A:matrix_coupling}, we obtain:
\begin{align}
\omega_t & = \frac{1 - \gamma_t\lambda_t}{1-\beta_\lambda} \nonumber\\
\implies \lambda_t &= \frac{\beta_\lambda \omega_t + (1 - \omega_t)}{\gamma_t} = \frac{1 - \omega_t(1-\beta_\lambda)}{\gamma_t} = \frac{1 - \omega_t + \omega_t\beta_\lambda}{\gamma_t}\nonumber
\end{align}
The constant discount weighting can be recovered by making $\beta_\lambda = \gamma$ for constant $\gamma$.
}
\subsubsection{Stability \& convergence}{
For the on-policy case we can show that this coupling between the weighting and the trace decay parameter is sufficient for stability. Moreover, because the traces are on-policy and they do not have importance sampling ratios, the variance is always finite, so the coupling also ensures convergence, not just stability of the value learning process.

Let the value function $\V_\w$ be a linear function of the form: $\V_\w(s) = \w^\top \x(s)$ with $\w \in \R^n$ learnable parameters, where $\x : \S \rightarrow \R^n$ is a feature mapping. Let $\X \in \R^{|\S|\times n}$ be the feature matrix whose rows are the vectors $\x(s)$ for different states $s$.

Consider the semi-gradient learning update for the selective  TD($\lambda, \omega$) algorithm:
\begin{align}
    \w_{t+1} &= \w_{t} + \alpha_{t} \tilde{\e}_t (\underbrace{R_{t+1} \!+\! \gamma_t \V_\w(S_{t\!+\!1}) \!-\! \V_\w(S_t)}_{\delta_t}) \label{eq_A:TD_lambda_update}\,, 
\end{align}
with $\tilde{\e}_{t} \!=\! \gamma_t \lambda_t \tilde{\e}_{t-1} \!+\! \omega_t \nabla\V_\w(S_{t})$, and $\delta_t$ shorthand for the TD error, $\tilde{\e}_t$ the selective instantaneous eligibility trace, $\alpha_t$ is the step size parameter, $\omega : \S \to [0, 1]$ is the weighting function, $\lambda : \mathscr{S} \to [0, 1]$ is the trace-decay function, $\gamma: \mathscr{S} \to [0, 1]$ is the temporal discounting function. Additionally, we assume the following mild conditions: 
\begin{enumerate}
    \item the state space is finite
    \item the feature function $\x : \S \to \mathbb{R}^n$ s.t. the $\X \in \mathbb{R}^{|S|\times n}$ has
linearly independent columns, with bounded variance;
    \item the rewards are bounded;
    \item the step-size sequence satisfies the Robbins-Monro conditions \cite{robbins1951stochastic}: $\sum_{t=0}^{\infty}\alpha_t = \infty$ and $\sum_{t=0}^{\infty}\alpha_t^2 < \infty$;
    \item $\gamma : \S \to [0, 1]$ s.t. $\prod_{k=1}^{\infty} \gamma(S_{t+k}) = 0$, w.p. $1$, $\forall t > 0$;
    \item $\omega : \S \to [0, 1]$ s.t. $\prod_{k=1}^{\infty} \omega(S_{t+k}) \neq 0$, w.p. $1$, $\forall t > 0$;
    \item experience is sampled on-policy from the Markov chain ($S_t, A_t, R_{t+1}, S_{t+1}) \sim d_\pi(S_t) \pi(A_t|S_t) P(S_{t+1}|S_t, A_t)$ with stationary distribution $d_\pi$.
\end{enumerate}
\begin{proposition}\label{prop_A:i_lambda}
For $\V_\w(s) = \w^\top \x(s)$, the semi-gradient update of selective TD($\lambda, \omega$) in Eq.~\eqref{eq_A:TD_lambda_update} with $\omega(s) = 1 - \gamma(s)\lambda(s)$ converges to the fixed point:
\begin{align}
\w^* = \E\left[\left(\sum_{k=0}^{t}\gamma_{t-k}^{(k)}\lambda_{t-k}^{(k)}\omega_{t-k}\x_{t-k}\right)\left(\gamma_{t+1}\x_{t+1} - \x_t\right)^\top\right]^{-1} \E\left[\left(\sum_{k=0}^{t}\gamma_{t-k}^{(k)}\lambda_{t-k}^{(k)}\omega_{t-k}\x_{t-k}\right)R_{t+1}\right]\,,
\end{align}
with $\gamma_t^{(k)} = \prod_{j=t+1}^{t+k} \gamma_j$, $\lambda_t^{(k)} = \prod_{j=t+1}^{t+k} \lambda_j$.
\end{proposition}


\begin{proof} 
Firstly, under the assumptions 1-4 above, the stochastic algorithm TD($\lambda, \omega$) (Eq.~\eqref{eq_A:TD_lambda_update}) behaves like the corresponding expected update equation under the on-policy stationary distribution $d_\pi$:
\eqq{
{\w}_{t+1} &= {\w}_{t} + \alpha_{t}\left(\E_{d_\pi}[\tilde{\e}_t R_{t+1}] - \E_{d_\pi}[\tilde{\e}_t (\x_t - \gamma_t \x_{t+1})^\top] {\w_t}\right), \label{eq_A:expected_update}
}
and ${\w}_t$ converges with probability 1 to ${\w}^*$, the solution to the expected update equation \eqref{eq_A:expected_update}.

\paragraph{Stability}{
We begin by showing stability of TD($\lambda, \omega$).
Let $Z_t = (S_t, A_t, \tilde{\e}_t)$ for $t \geq 0$ be the Markov chain resulting from adding $\tilde{\e}$ to the stationary Markov chain  $\{(S_t, A_t)\}_{t=0}^{\infty}$ with transition probabilities given by $P_\pi$, s.t.:
\eq{
\tilde{\e}_t = \sum_{k=0}^{t}\left(\prod_{j=0}^{k-1}\gamma_{t-j}\lambda_{t-j}\right)\omega_{t-k}\x_{t-k}\,.
    }
Since $\tilde{\e}_t$ and $\x_{t+1}$ are deterministic functions of $(S_t, A_t)$ and the distribution of $s_{t+1}$ only depends of $s_t$, the resulting chain $Z_t$ is Markov. Let $\E_{d_{\pi}}[\cdot]$ denote the expectation with respect to the steady state distribution $d_\pi$.

Let $\A(Z_t) = \tilde{\e}_t (\x_t - \gamma_t \x_{t+1})^\top$, $\b(Z_t) = \tilde{\e}_t R_{t+1}$ and $\A = \E_{d_\pi}[\tilde{\e}_t (\x_t \!-\! \gamma_t \x_{t+1})^\top]$, $\b = \E_{d_\pi}[\tilde{\e}_t R_{t+1}]$.
}
The fixed point equation of the deterministic system is:
\eq{
{\w}_{t+1} \!=\! (\I - \alpha_{t}\A){\w}_{t} + \alpha_{t}\b\! \label{eq_A:fixed_point_eq}
}

Since these methods are not true gradient methods, the asymptotic behaviour of any TD algorithm generally depends on a stability criteria that requires the eigenvalues of the iteration matrix $\A$ have positive real components \cite{sutton2016emphatic}.
The stochastic algorithm converges if and only if the deterministic algorithm converges \cite{sutton1988learning} and if both algorithms converge, they converge to the same fixed point.

Let $\bLambda$ and $\bGamma$ be matrices with diagonal entries corresponding to the functions $\lambda$ and $\gamma$. 
Let 
$$\P^\lambda_\pi = (\I - \P_\pi \bGamma\bLambda)^{-1}\bGamma (\I - \bLambda)\P_\pi\,,$$
such that:
\eqq{
(\I  - &\P_\pi\bGamma\bLambda)^{-1}(\I - \P_\pi\bGamma) =\\
&= (\I - \P_\pi\bGamma\bLambda)^{-1}(\I -\P_\pi\bGamma\bLambda + \P_\pi\bGamma\bLambda - \P_\pi\bGamma)\\
&=(\I -\P_\pi\bGamma)^{-1}(\I - \P_\pi\bGamma\bLambda + \P_\pi\bGamma(\bLambda - \I)) \\
&= \I - (\I - \P_\pi \bGamma\bLambda)^{-1}\bGamma (\I - \bLambda)\P_\pi\\
&= \I - \P^\lambda_\pi \label{eq_A:P_lambda}\,.
}

Expanding the $\A$ and $\b$ matrices we have:
\eq{
\A &=\X^\top \underbrace{ \tilde{\D}_\pi (\I - \P_\pi \bGamma \bLambda)^{-1} (\I -  \P_\pi \bGamma)}_{\K} \X,\\
\b &= \X^\top \tilde{\D}_\pi(\I - \P_\pi \bGamma \bLambda)^{-1} \r,
}
where $\tilde{\D}_{\pi}$ is a diagonal matrix with elements $d_{\pi}(s) \omega(s)$ on the diagonal. Following \citet{sutton2016emphatic}, we refer to $\K$ as the \emph{``key matrix''}.

To ensure convergence regardless of the representation function, we use the assumption that $\X$ is full rank and require \citealp[cf.][]{sutton2016emphatic} that:
\begin{enumerate}
    \item the diagonal entries of $\K$ are non-negative
    \item the off-diagonal entries are non-positive
    \item the row sums are non negative
    \item the columns sums are positive
\end{enumerate}

Conditions (1-3) follow from Lemma 4 of \citet[][]{white2017unifying} and the fact that $\tilde{\D}_\pi$ is a non-negative diagonal weighting matrix.

For the last condition, if we assume $\omega(s) > 0, \forall s\in\S$, then, similarly to \citep{sutton2016emphatic}, we have:
\eq{
\mathbf{1}^\top \K &= \mathbf{1}^\top\tilde{\D}_\pi \left(\I \!-\!  \P_\pi \bGamma \bLambda\right)^{-1} (\I - \P_\pi\bGamma) \\
&\text{(using $\mathbf{\Omega} = \I - \bLambda\bGamma$)}\\
&=\mathbf{1}^\top \D_\pi (\I \!-\!  \bGamma\bLambda) \left(\I \!-\! \P_\pi \bGamma\bLambda\right)^{-1}(\I -\P_\pi\bGamma) \\
&= (\d_\pi^\top \!-\! \d_\pi^\top \bGamma\bLambda ) \left(\I \!-\!  \P_\pi\bGamma \bLambda\right)^{-1}(\I - \P_\pi\bGamma) \\
&\text{(using $\d_\pi^\top = \d_\pi^\top \P_\pi$)}\\
&= (\d_\pi^\top \!-\! \d_\pi^\top \P_\pi\bGamma\bLambda ) \left(\I \!-\!  \P_\pi \bGamma\bLambda\right)^{-1}(\I - \P_\pi\bGamma) \\
&= \d_\pi^\top(\I \!-\! \P_\pi\bGamma\bLambda ) \left(\I \!-\!  \P_\pi \bGamma\bLambda\right)^{-1}(\I - \P_\pi\bGamma) \\
&= \d_\pi^\top(\I - \P_\pi\bGamma) \\
&\text{(using $\d_\pi^\top = \d_\pi^\top \P_\pi$)}\\
&= \d_\pi^\top(\I - \bGamma) > 0
}
Using $\omega(s) = 1 - \gamma(s) \lambda(s)$, all components of the column sums become positive. Thus, the key matrix $\A$ is positive definite and the selective TD($\lambda, \omega$) algorithm is stable.

If $\exists s\in\S$ s.t. $\omega(s) = 0$, then we can set $\gamma(s) = \lambda(s) = 1$, which induces a new super-imposed MDP, with temporally-extended dynamics comprising of multi-step transition dynamics and multi-step cumulated rewards of the original MDP. The new MDP is then stable by invoking the argument proved above.

\textbf{Convergence of TD($\lambda_t, \omega_t$)}. Stability is a prerequisite for full convergence of the
stochastic algorithm. For full convergence, we can apply Theorem 2 from \citet{Tsitsiklis_VanRoy:97}, adapted and restated below.

\paragraph{Theorem 2 from \citet{Tsitsiklis_VanRoy:97}}{
Consider an iterative algorithm of the form:
\eq{
\w_{t+1} = \w_t + \alpha_t (-A(Z_t) \w_t + b(Z_t)),
}
where:
\begin{enumerate}
    \item the (predetermined) step-size sequence $\alpha_t$ is positive, non-increasing, and satisfies $\sum_{t=0}^{\infty} \alpha_t = \infty$ and $\sum_{t=0}^{\infty} < \infty$;
    \item $Z_t$ is a Markov process with a unique invariant distribution, and there exists a mapping $h$ from the states of the Markov process to the positive reals, satisfying the remaining conditions. Let $\E_{d_\pi}[\cdot]$ stand for the expectation with respect to this invariant distribution;
    \item $\A(\cdot)$ and $\b(\cdot)$ are matrix and vector valued functions, respectivly, for which $\A = \E_{d_\pi}[\A(Z_t)]$ and $\b = \E_{d_\pi}[\b(Z_t)]$ are well-defined and finite;
    \item the matrix $\A$ is positive definite;
    \item there exist constants $C$ and $q$ such that for all $Z$:
    \eq{
    \sum_{t=0}^{\infty} \| \E[\A(Z_t)| Z_0 = Z] - \A\| \leq C(1 + h^q(Z))\, \text{and}\\
    \sum_{t=0}^{\infty}\| \E[\b(Z_t) | Z_0 = Z] - b\| \leq C(1 + h^q(Z))\,;
    }
    \item for any $q > 1$ there exists a constant $\mu_q$ such that for all $Z, t$
    \eq{
    \E[h^q(Z_t)|Z_0 = Z] \leq \mu_q (1+h^q(Z)).
    }
\end{enumerate}
Then, $\w_t$ converges to $\w^*$, with probability 1, where $\w^*$ is the unique vector that satisfies $\A \w^* = \b$.
}

The assumptions of Theorem 2 hold in our case since the last two remaining conditions (v) and (vi), stating that the dependence of $\A(Z_t)$ and $\b(Z_t)$ on $Z_k, \forall k \leq t$ is exponentially decreasing, are satisfied by the fact that the trace iterates have bounded variance and the fact that $\{Z_t\}_{t=0}^{\infty}$ is Markov, \citealp[cf. Assumption 3,][]{Tsitsiklis_VanRoy:97}. The first follows from the fact that the range of $\omega$ and $\lambda$ is $[0, 1]$, and the second by definition of the trace.


\textbf{Fixed point.} We now examine the fixed point of the system:
\eq{
\E[\delta_t \tilde{\e}_t] = \E[\tilde{\e}_t(R_{t+1} + \gamma_t\x_{t+1}^\top\w - \x_t^\top\w)] = 0\,,
}
Unfolding the trace, we have:
\eq{
\tilde{\e}_t &= \gamma_t\lambda_t\tilde{\e}_{t-1} + \omega_t\x_t\\
&=\omega_t \x_t + \gamma_t \lambda_t(\gamma_{t-1}\lambda_{t-1}\tilde{\e}_{t-2} + \omega_{t-1}\x_{t-1})\\
&=\sum_{k=0}^{t}\left(\prod_{j=0}^{k-1}\gamma_{t-j}\lambda_{t-j}\right)\omega_{t-k}\x_{t-k}
}
which results in:
\eq{
\w^* &=\!\E\left[\left(\sum_{k=0}^{t}\gamma_{t-k}^{(k)}\lambda_{t-k}^{(k)}\omega_{t-k}\x_{t-k}\right)\left(\gamma_{t+1}\x_{t+1} \!-\! \x_t\right)^\top\right]^{-1}\E\left[\left(\sum_{k=0}^{t}\gamma_{t-k}^{(k)}\lambda_{t-k}^{(k)}\omega_{t-k}\x_{t-k}\right)R_{t+1}\right]\,,
}
with $\gamma_t^{(k)} = \prod_{j=t+1}^{t+k} \gamma_j$, $\lambda_t^{(k)} = \prod_{j=t+1}^{t+k} \lambda_j$.
\end{proof}
}

\section{Details on empirical illustrations} 
\subsection{Weightings for off-policy distribution correction}\label{A:Weightings for off-policy distribution correction}

\begin{figure}[t]
\begin{center}
        \includegraphics[scale=.22]{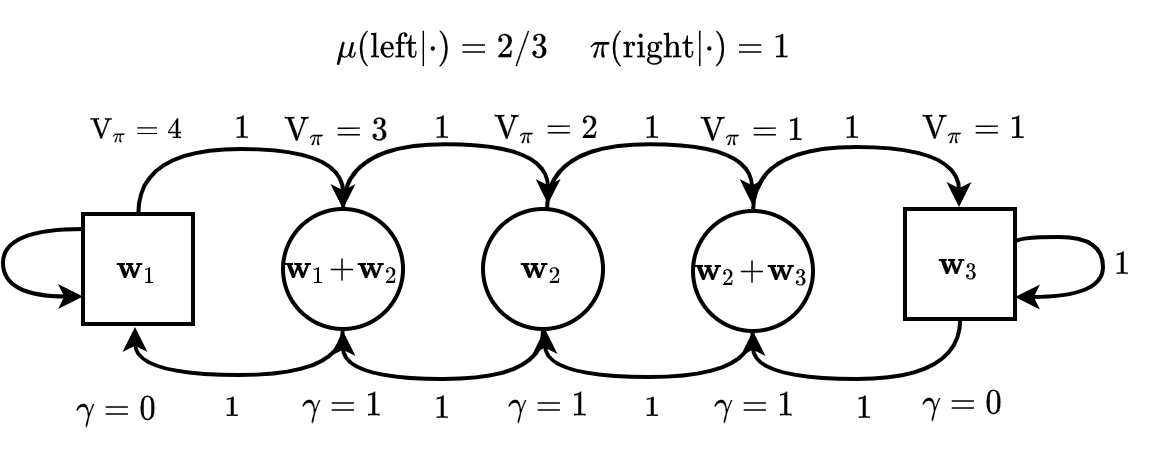}
    \end{center}
\caption{\label{fig:5state_mdp_app}
\small \textbf{The 5-state chain MDP} (cf. \citet{sutton2016emphatic}): The states shown as square cause soft termination, with
$\gamma(s) = 0$, but they do not restart the agent. There are two actions, ``left'' and ``right'', which deterministically cause transitions to the left or right except at the ends of the chain, where there may be a self-transition. The reward on all transitions is $+1$. The behaviour policy selects ``left''  2/3rds of the time in all states, which causes more time to
be spent in states on the left than on the right. The true value function $\V_\pi$ is depicted above each state. The notation $\w_i$ describes state aliasing in the observations, denoting that the $i$-th component of the current parameter vector is $\w_i$. Since there are five states and only three parameters, it is impossible, to represent the true $\V_\pi$ exactly.
}
\end{figure}
\paragraph{Experimental setup}{
We used the $5$-state MRP denoted in Fig.~\ref{fig:5state_mdp_app} to illustrate the following policy-evaluation algorithms: Off-policy TD, ETD, X($1$)-ETD, and X($0$)-ETD, described next.
}

\paragraph{Algorithms \& hyperparameters}{
The Off-policy TD algorithm is the canonical off-policy evaluation algorithm \cite{precup2001off}, without any prior distribution corrections in the form of stationary distribution ratios between the target policy and the behaviour policy; it only uses importance sampling ratios $\rho$.

The ETD algorithm is the one introduced in \cite{sutton2016emphatic}. 

For the two emphatic algorithms that use expected traces, we differentiate with $\eta^f$ (cf. \cref{A:Weightings for distribution correction: off-policy expected emphasis}) the following instances of the generic algorithm X($\eta_f$)-ETD:
\begin{itemize}
    \item X($1$)-ETD -- learns the follow-on trace with Monte-Carlo regression on the follow-on trace, cf. Eq.~\eqref{eq_A:followon_mixture} with $\eta_f$ = 1;
    \item X($0$)-ETD -- learns the follow-on by backward TD, i.e. the follow-on trace model bootstraps on itself, cf. Eq.~\eqref{eq_A:followon_mixture} with $\eta_f = 0$.
\end{itemize} 
All the algorithms are compared for $\lambda = 0$. We use no other internal discount factor beside the chain's own termination function $\gamma$. The step-sizes for the value function, for all algorithms, are decayed with $1/t^d$, where $t$ is the time-step, and $d$ is chosen from $\{0.4, 0.5, 0.8, 0.9\}$, with the best values obtained through hyperparameter search: (i) Off-policy TD: $d = 0.5$, (ii) ETD: $d = 0.9$, (iii) X($1$)-ETD: $d = 0.9$, (iv) X($0$)-ETD: $d = 0.9$. 

For the two algorithms that use expected emphasis, learning rates for the model $f_\vphi$ that estimates the follow-on trace $F$, use the same step-size decay schedule $1/t^{d_f}$, with values for $d_f$ chosen from the same interval, and the best value obtained for both algorithms $d_f = 0.5$. The model is a linear function approximator on the observations, followed by a ReLU non-linearity to keep the output positive (the latter is not important, as similar results can be obtained without it).
}
\subsection{Weightings for on-policy distribution correction -- Atari Ms.Pac-Man experiments}\label{A:Weightings for on-policy distribution correction}
For our deep reinforcement learning experiments on Atari games, we used the canonical Ms.Pac-Man to illustrate the importance of the connection between selectivity $\omega$, trace-decay $\lambda$ and trace-bootstrapping $\eta$. We start by describing the experimental setup, followed by the online selective Q($\lambda, \omega$) algorithms, after which we move on to the expected traces versions, namely QET($
\lambda, \eta, \omega$).

\paragraph{Experimental setup}{
All the Atari experiments were run with the ALE \cite{ale}, similarly to \citet{van2020expected}, including using action repeats (4x), but without downsampling (or framestacking, or pooling) the observation stream.
With probability $\epsilon$, the agent's observation is replaced with random standard Gaussian noise, to mimic a noisy observation sensor. To simulate access to a module that detects such noisy observations, we provide access to a time-dependent interest $i_t$, capturing whether an observation is noisy or not, s.t. $i_t = 0$ if the observation at time step $t$ is noisy, and $i_t=1$ otherwise. We use no other corrections, e.g. corrections to the discrepancy between the behaviour $\epsilon$-greedy and the target greedy policy, as we found those to not be useful in this setting, resulting in $\omega_t \equiv i_t$.

 We keep the discount factor $\gamma = 0.99$ constant. We report the mean return while training on $20$M frames.
 
 In all cases, we used $\epsilon$-greedy exploration (cf. \cite{sutton2018}), with an $\epsilon = 1e-2$, which we do not decay. 

We use a similar setting to \citet{van2020expected}, with the following modifications.
We apply a different feature extraction network, cf. \cite{inductive_bias}; particularly, we use $2$ convolutional layers, with $5\times 5$ kernels, stride $5$, and $64$ channels, followed by ReLU nonlinearities. The output is then passed through a $512$ linear layer, with ReLU activation. 

These experiments were conducted using Jax \cite{jax2018github}, Haiku \cite{haiku2020github} and Optax \cite{optax2020github}.
}

\begin{algorithm}[t]
  \caption{Selective Credit Assignment - QET($\lambda, \eta, \omega)$}
  \label{alg:selective_q_lambda}
\begin{algorithmic}
  \STATE {\bfseries Initialize: } $\text{policy } \pi, \w, \T, \omega, \gamma, \lambda, \eta$
  \\
  \STATE $S \sim \text{env}()$
    \FOR{$\text{each interaction } \{1,2 \dots T\}$}
        \STATE $A \sim \pi(S), R, \gamma, S^\prime \sim \text{env}(A)$
        \STATE $\T \leftarrow \T + \alpha^{\T}\frac{\partial{\tilde{\z}_\T(S)}}{\partial{\T}}(\gamma \lambda(S)\tilde{\e} - \tilde{\z}_\T(S))$
        \STATE $\tilde{\e} \leftarrow \eta_t \gamma \lambda(S) \tilde{\e} + (1 - \eta_t)\tilde{\z}_{\T}(S) + \omega(S) \nabla \Q_\w(S,A)$
        \STATE $R^\lambda = R + \gamma (1 - \lambda(S^\prime)) \Q_\w(S^\prime, a^\prime)$
        \STATE $\Delta^{\Q}= R^\lambda \tilde{\e}  - \Q_\w(S,A) \nabla \Q_\w(S,A) \omega(S)$
        \STATE $a^\prime = \arg \max_a \Q_\w(S^\prime, a)$
        \STATE $\w \leftarrow \w + \alpha^{\w}\Delta^{\Q}$
        \STATE \textbf{if} $S$ is terminal: $S \sim \text{env}()$; $\tilde{\e} \leftarrow 0$
    \ENDFOR
\end{algorithmic}
\end{algorithm}

\paragraph{Algorithms evaluated}{
We start by describing the baseline algorithm Q($\lambda$), followed by the algorithms that specifically apply selectivity.

\textbf{Q($\lambda$)}\quad
In the algorithm~\ref{alg:selective_q_lambda}, we have $\eta_t = 1\forall t$, i.e. we only use accumulating traces, so we omit line $5$, since this algorithm does not use the expected traces $\tilde{\z}_\T$. For each transition, we first decay the trace $\tilde{\e}$ and then update it using line $6$. We further compute the finite-horizon one-step return $R^\lambda$ (line $7$), where $\gamma = 0$ on termination (and then $S^\prime$ is the first observation of the next episode). Instead of the usual SGD algorithm illustrated for simplicity in line $10$, for training, we use ADAM \cite{adam}; we learn the value function with momentum$=0.9$, which simulates soft-batching; we set the other parameters from ADAM to $b_1 = 0.99$ (the exponential decay rate to track the first moment of past gradients) and $b_2 = 0.9999$ (the exponential decay rate to track the second moment of past gradients), $\epsilon = 1e-4$ (the small constant applied to denominator outside the square root--as in \cite{adam}, to avoid dividing by zero when rescaling). We use a step-size of $\alpha^{\w} = 1e-5$ for learning the value function.


\textbf{Variations of Q($\lambda, \omega$)}\quad
The algorithms used in the experiments are all variations of Q($\lambda, \omega$). We label Q($\lambda$), the default baseline algorithm that uses $\lambda=0.9$ and uniform weightings over the state space: $\omega = \omega_t = 1, \forall t$. We use Q($\lambda, \omega_t$) (with the \emph{``$t$''} subscript denoting state or time-dependence) for the algorithm that uses $\lambda=0.9$, but uses the ground-truth interest to set the weighting $\omega_t$ ($1$ for non-noisy states, and $0$ otherwise). 
Lastly, Q($\lambda_t, \omega_t$) denotes the algorithm using, in addition to the ground truth interest in setting $\omega_t$, also Eq.~\eqref{eq:coupling_2} to set $\lambda_t$.

\textbf{QET($\lambda, \eta, \omega$)}\quad
The expected-traces algorithm is similar to Q($\lambda, \omega$), except we now use the expected traces $\tilde{\z}_\T$ in place of the instantaneous traces, so we update the parameters $\T$ as well, in addition to $\w$, using ADAM, with the same hyperparameters as for the value function: momentum, $b_1$, $b_2$, and $\epsilon$.
We use the step size $\alpha^{\T} = 1e-2$ for learning the expected traces.

Similarly to \citet{van2020expected}, we also split the computation of $\Q(s, a)$ into two separate parts, such that $\Q(\w,\zeta)(s, a) = \w_a^\top \mathbf{x}_\zeta(s)$. This separation is just so that we can keep labeled separate subsets of parameters as $(\w, \zeta)$ rather than merging all of them into a single vector $\w$, using $\mathbf{x}(s)$ to denote the last hidden layer of the feature extraction part of the network, on top of which the last linear layer of the q-function is applied. We keep separate traces for these subsets, and we just apply accumulating instantaneous traces to the feature extraction network, similarly to \citet{van2020expected}. This separation is equivalent to keeping one big trace for the combined set. 
We refer the reader to \citet{van2020expected} for more details on this particularity. 
The motivation for this split in parameters is to avoid learning an expected trace for the full trace, which has millions of elements.  Instead, in practice, we only learn expectations for traces corresponding to the last layer. 

The difference from how this algorithm is presented in \citet{van2020expected} is that we only condition the function $z_\T(s)$ on the state, and not the action, due to our new derivation of the algorithm in \cref{A:New derivation of Off-policy}. 

As customary, we do not backpropagate the gradient coming from the expected traces' loss further into the feature representation.

\textbf{Variations of QET($\lambda, \eta, \omega$)}\quad
We now describe the variations of QET compared in the experiment. 
For the first versions of the algorithms, we use $\eta=0, \forall t$, i.e. using expected traces everywhere, instead of the instantaneous counterparts.
 The algorithms QET($\lambda, \eta$)-baseline, QET($\eta, \lambda, \omega_t$), QET($\eta, \lambda_t, \omega_t$) are analogous to Q($\lambda$)-baseline, Q($\lambda, \omega_t$), and QET($\lambda_t, \omega_t$), respectively, which were described in the previous section. The algorithm QET($\eta_t, \lambda_t, \omega_t$), in addition to using Eq.~\eqref{eq:coupling_2} for setting $\lambda_t$, and the ground truth interest for setting $\omega_t$, also uses Eq.~\eqref{eq:coupling_3} to set $\eta_t$, i.e. it uses the expected trace more in states where the selectivity weighting $\omega$ is higher, and the instantaneous traces more when the weighting is lower. Lastly, QET($\eta_t, \lambda_t, \omega_t, \tilde{\Delta}^{\z}_t$), in addition, uses the coupling Eq.~\eqref{eq:coupling_4} to learn expected traces constrained by selectivity.
}

\subsection{Off-policy counterfactual evaluation}
\label{A:Off-policy counterfactual evaluation}

\paragraph{Discussion on learning with function approximation}{
With function approximation, learning expected traces off-policy can be problematic. Particularly, one can learn selective expected eligbility traces $\tilde{\z}_\T$ using Monte-Carlo methods, by regressing on $\gamma_t \rho_t \lambda_t \tilde{\e}_{t-1} + \omega_t \nabla\V_{\w_t}(S_t)$, but the traces can have very high variance resulting from the product of importance sampling ratios. On the other hand, learning the traces with backward TD can easily diverge due to ``off-policiness'', from the same reasons the value learning process can diverge. We can interpolate between Monte-Carlo methods and backward TD using \emph{selective mixture traces}, similarly to regular mixture traces,  using a different mixture parameter $\tilde{\eta}$, s.t.:
\eqq{
\tilde{\e}^{\tilde{\eta}}_t = (1 - \tilde{\eta}) \tilde{\z}_{\T_t}(S_t) + \tilde{\eta}(\rho_t\gamma_t\lambda_t \tilde{\e}^{\tilde{\eta}}_{t-1} + \omega_t\nabla \V_{\w_t}(S_t))\,, \label{eq_A:mixture_selective}
} 
and updating $\T$ with
\eq{
\Delta^{\z}_t &\equiv \frac{\partial {\tilde{\z}}_{\T_t}(S_{t})}{\partial \T_t}\Big(\rho_t\gamma_t\lambda_t\tilde{\e}^{ \tilde{\eta}}_{t-1}  + \omega_t \nabla \V_{\w_t}(S_t)-\tilde{\z}_{\T_t}(S_{t})\Big)\,,
\\
\T_{t+1} &= \T_{t} + \alpha^{\T}_t \Delta^{\z}_t  \,,
}
where $\alpha^{\T}_t\in (0,1)$ is a step-size parameter and $\tilde{\eta}$ is analogous to $\lambda$, but for learning selective traces here, rather than values, using TD methods when $\tilde{\eta} = 0$, and Monte-Carlo regression when $\tilde{\eta}=1$. 

To stabilize learning, we could choose an intermediary value for $\tilde{\eta}$, that achieves an optimal balance between bias and variance. Moreover, we can also  stabilize the learning process by instantiating $\omega$ using emphatic weightings, which can guarantee convergence for linear function approximation. 
}

We now discuss the details regarding the empirical illustration in the Open world gridworld domain we used. This illustration is meant to illustrate the learning process, so we do not use function approximation, therefore we omit any non-uniform weightings.

\paragraph{Experimental setup}{
In Fig.~\ref{fig:envs}-Left, there are two goals depicted with ``G'', giving the same reward of $10 / \epsilon_r$, with probability $\epsilon_r$, otherwise $0$. All other rewards are $0$. When the agent reaches the goal, the episode restarts with a random initial location sampled from the state space.

The agent learns different expected traces for the two policies transitioning it to each of the two corners of the world: top-right and bottom-left (we may call those loosely options), while following a random behaviour policy $\mu$. 

To increase the difficulty in learning, we let each option's policy take a random action with probability $\epsilon_o = 0.2$, and let the environment transition the agent randomly in a cardinal direction with probability $\epsilon_p = 0.05$. We use one-hot state representations, so we omit emphatic weightings, and just use importance sampling ratios between each option's policy $\pi_o$ and the behaviour policy $\mu$, when learning the expected trace of each option. 
}
\paragraph{Algorithms \& hyperparameters}{We compare the off-policy TD($\lambda$) and ET($\lambda, \eta$) with uniform weightings. We clip the importance sampling ratios to $1$. We learn the traces with Monte-Carlo regression, using $\tilde{\eta}= 1$ in Eq.~\eqref{eq_A:mixture_selective}. The discount is $\gamma = 0.99$, except at the goals, where it is $0$. We set $\lambda = 0.98$. We decay step-sizes with $1/t^d$, with $d \in \{0.5, 0.7, 0.9, 1.\}$, the best values obtained through hyperparameter search: $d = 0.5$, for $\epsilon_r = 1.$, and $d = 0.9$, for $\epsilon_r = 0.001$. For the expected traces algorithm, we similarly decay the step size for the trace learning process, using $d_z = 0.001$, searched over values $\{0.1, 0.01, 0.001\}$. We learn both the value function and the traces with SGD, starting from a step size of $1$.
}

\begin{figure}
\begin{center}
        \includegraphics[scale=.335]{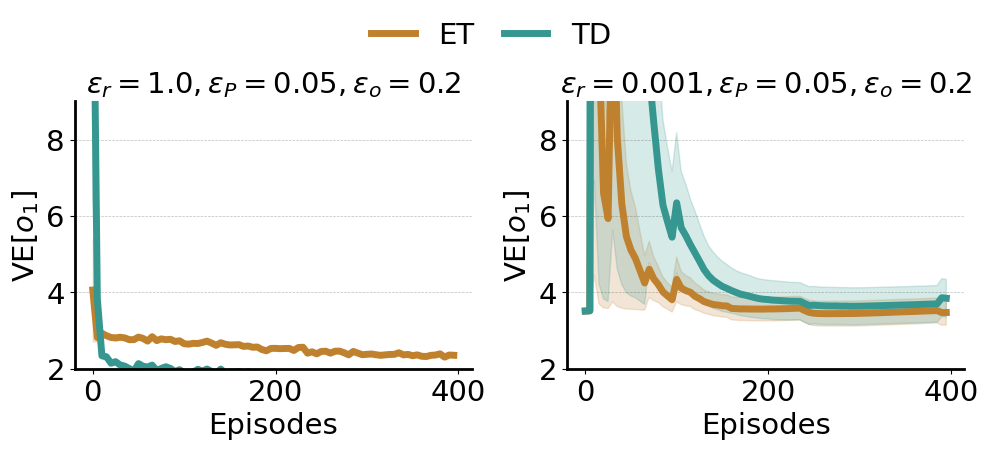}
    \end{center}
\caption{\textbf{Open World}: Value error for policy evaluation using {\color{turquoise}TD($\lambda=0.98$)} and {\color{brown}ET($\lambda=0.98$)}, for option $o_2$. 
Increasing the sparsity of the reward signal (left vs right plot), we observe policy-conditioned expected traces coped better with sparse rewards. Shaded areas show standard errors over $20$ seeds.
\label{fig:room_o2}
\small }
\end{figure}

\paragraph{Additional results}{
Fig.~\ref{fig:room_o2} illustrates the performance for the option with policy transitioning toward the bootom corner---$o_2$ (with similar results illustrated for the other policy---$o_1$, in the main text).}

\subsection{Sparse expected eligibility traces}\label{A:Empirical Sparse expected eligibility traces}
To illustrate the idea of sparse expected traces that can flow credit over the state space using temporally-extended actions, we use the following experimental setup.

\paragraph{Experimental setup}{
We use the Four Rooms domain from figure Fig.~\ref{fig:envs}-Right. We pre-learn the option policies using pre-defined interest and sub-goal functions, such that the options learn policies as illustrated in Fig.~\ref{fig:envs}-Right. The agent receives a reward of $r=10. / \epsilon_r$ with probabiltiy $\epsilon_r$ at the goal, depicted with ``G'' in Fig.~\ref{fig:envs}-Right. The discount is $\gamma=0.98$ everywhere, except at the goal where it is $0$, after which the episode restarts and the first observation of the new episode is sampled from one of the hallways.
}

\paragraph{Algorithms \& hyperparameters}{ 
The baselines we compare against are the default Q$_A$($\lambda=0.9$) and QET$_A$($\lambda=0.9, \eta=0$), which are defined over the primitive action-space. 

For the sparse learning algorithms, we assume the policy over actions is defined using a pre-specified weighting function $\omega$, s.t. $\omega(s) = 1$, if $s$ is a hallway, and $0$ otherwise. The policy over options uses only the pre-learned option space, without any primitive actions, which is sufficient for this illustration since the start states are in hallway locations, and so is the goal. 

For Q$_O$($\lambda_t, \omega_t$) and QET$_A$($\lambda_t, \eta_t, \omega_t$) (with the ``t'' subscript indicating state or time-dependence), we use Eq.~\eqref{eq:coupling_2} for setting $\lambda_t$, with $\beta_\lambda=0.9$, and similarly Eq.~\eqref{eq:coupling_3} for setting $\eta_t$, with $\beta_\eta = 0.9$. We also learn the traces themselves using Eq.~\eqref{eq:coupling_4}.
We decay all learning rates for the value function using $1/t^d$, searching over $d \in \{0.1, 0.2, 0.5, 0.7, 0.9, 1.\}$. We find the best parameters are $d=0.2$ for all algorithms. For the expected traces algorithms, we use the same scehdule $1/t^{d_z}$, with $d_z \in \{1., 0.1, 0.01, 0.01\}$, finding the best value is $d_z = 0.1$, for both algorithms. We learn both the value function and the traces with SGD, starting from a step size of $1$. For exploration we use $\epsilon$-greedy policies with $\epsilon = 0.1$. The options' policies are learned with intra-option discount factor $\gamma_O = 0.9$ and without randomness $\epsilon_O = 0.$, i.e. greedy policies. The trace learning process uses Monte-Carlo regression, i.e. $\tilde{\eta} = 1$ in Eq.~\eqref{eq_A:mixture_selective}. We do not use importance sampling ratios, or any other corrections.
}

%% file: main.bbl
\begin{thebibliography}{50}
\providecommand{\natexlab}[1]{#1}
\providecommand{\url}[1]{\texttt{#1}}
\expandafter\ifx\csname urlstyle\endcsname\relax
  \providecommand{\doi}[1]{doi: #1}\else
  \providecommand{\doi}{doi: \begingroup \urlstyle{rm}\Url}\fi

\bibitem[Anand and Precup(2021)]{anand2021preferential}
N.~Anand and D.~Precup.
\newblock Preferential temporal difference learning, 2021.

\bibitem[Arjona-Medina et~al.(2019)Arjona-Medina, Gillhofer, Widrich,
  Unterthiner, and Hochreiter]{ArjonaMedina2019RUDDERRD}
J.~A. Arjona-Medina, M.~Gillhofer, M.~Widrich, T.~Unterthiner, and
  S.~Hochreiter.
\newblock Rudder: Return decomposition for delayed rewards.
\newblock \emph{ArXiv}, abs/1806.07857, 2019.

\bibitem[Bacon et~al.(2017)Bacon, Harb, and Precup]{Bacon2017TheOA}
P.~Bacon, J.~Harb, and D.~Precup.
\newblock The option-critic architecture.
\newblock \emph{ArXiv}, abs/1609.05140, 2017.

\bibitem[Bellemare et~al.(2012)Bellemare, Naddaf, Veness, and Bowling]{ale}
M.~G. Bellemare, Y.~Naddaf, J.~Veness, and M.~Bowling.
\newblock The arcade learning environment: An evaluation platform for general
  agents.
\newblock \emph{CoRR}, abs/1207.4708, 2012.
\newblock URL \url{http://arxiv.org/abs/1207.4708}.

\bibitem[Bradbury et~al.(2018)Bradbury, Frostig, Hawkins, Johnson, Leary,
  Maclaurin, Necula, Paszke, Vander{P}las, Wanderman-{M}ilne, and
  Zhang]{jax2018github}
J.~Bradbury, R.~Frostig, P.~Hawkins, M.~J. Johnson, C.~Leary, D.~Maclaurin,
  G.~Necula, A.~Paszke, J.~Vander{P}las, S.~Wanderman-{M}ilne, and Q.~Zhang.
\newblock {JAX}: composable transformations of {P}ython+{N}um{P}y programs,
  2018.
\newblock URL \url{http://github.com/google/jax}.

\bibitem[Chelu et~al.(2020)Chelu, Precup, and {van
  Hasselt}]{chelu2020forethought}
V.~Chelu, D.~Precup, and H.~P. {van Hasselt}.
\newblock Forethought and hindsight in credit assignment.
\newblock \emph{Advances in Neural Information Processing Systems}, 33, 2020.

\bibitem[Hallak et~al.(2016)Hallak, Tamar, Munos, and Mannor]{hallak2016}
A.~Hallak, A.~Tamar, R.~Munos, and S.~Mannor.
\newblock Generalized emphatic temporal difference learning: Bias-variance
  analysis.
\newblock \emph{In Proceedings of the Thirtieth AAAI Conference on Artificial
  Intelligence (AAAI-16)}, 2016.

\bibitem[Harutyunyan et~al.(2019{\natexlab{a}})Harutyunyan, Dabney, Borsa,
  Heess, Munos, and Precup]{Harutyunyan2019TheTC}
A.~Harutyunyan, W.~Dabney, D.~Borsa, N.~Heess, R.~Munos, and D.~Precup.
\newblock The termination critic.
\newblock In \emph{AISTATS}, 2019{\natexlab{a}}.

\bibitem[Harutyunyan et~al.(2019{\natexlab{b}})Harutyunyan, Dabney, Mesnard,
  Azar, Piot, Heess, {van Hasselt}, Wayne, Singh, Precup, and
  Munos]{Harutyunyan2019HindsightCA}
A.~Harutyunyan, W.~Dabney, T.~Mesnard, M.~G. Azar, B.~Piot, N.~M.~O. Heess,
  H.~{van Hasselt}, G.~Wayne, S.~Singh, D.~Precup, and R.~Munos.
\newblock Hindsight credit assignment.
\newblock \emph{ArXiv}, abs/1912.02503, 2019{\natexlab{b}}.

\bibitem[Harutyunyan et~al.(2019{\natexlab{c}})Harutyunyan, Vrancx, Hamel,
  Nowe, and Precup]{pmlr-v97-harutyunyan19a}
A.~Harutyunyan, P.~Vrancx, P.~Hamel, A.~Nowe, and D.~Precup.
\newblock Per-decision option discounting.
\newblock In K.~Chaudhuri and R.~Salakhutdinov, editors, \emph{Proceedings of
  the 36th International Conference on Machine Learning}, volume~97 of
  \emph{Proceedings of Machine Learning Research}, pages 2644--2652. PMLR,
  09--15 Jun 2019{\natexlab{c}}.
\newblock URL \url{https://proceedings.mlr.press/v97/harutyunyan19a.html}.

\bibitem[Hennigan et~al.(2020)Hennigan, Cai, Norman, and
  Babuschkin]{haiku2020github}
T.~Hennigan, T.~Cai, T.~Norman, and I.~Babuschkin.
\newblock {H}aiku: {S}onnet for {JAX}, 2020.
\newblock URL \url{http://github.com/deepmind/dm-haiku}.

\bibitem[Hessel et~al.(2019)Hessel, van Hasselt, Modayil, and
  Silver]{inductive_bias}
M.~Hessel, H.~van Hasselt, J.~Modayil, and D.~Silver.
\newblock On inductive biases in deep reinforcement learning.
\newblock \emph{CoRR}, abs/1907.02908, 2019.
\newblock URL \url{http://arxiv.org/abs/1907.02908}.

\bibitem[Hessel et~al.(2020)Hessel, Budden, Viola, Rosca, Sezener, and
  Hennigan]{optax2020github}
M.~Hessel, D.~Budden, F.~Viola, M.~Rosca, E.~Sezener, and T.~Hennigan.
\newblock Optax: composable gradient transformation and optimisation, in jax!,
  2020.
\newblock URL \url{http://github.com/deepmind/optax}.

\bibitem[Hung et~al.(2019)Hung, Lillicrap, Abramson, Wu, Mirza, Carnevale,
  Ahuja, and Wayne]{Hung2019OptimizingAB}
C.-C. Hung, T.~P. Lillicrap, J.~Abramson, Y.~Wu, M.~Mirza, F.~Carnevale,
  A.~Ahuja, and G.~Wayne.
\newblock Optimizing agent behavior over long time scales by transporting
  value.
\newblock \emph{Nature Communications}, 10, 2019.

\bibitem[Jiang et~al.(2021)Jiang, Zhang, Chelu, White, and van
  Hasselt]{ray2021}
R.~Jiang, S.~Zhang, V.~Chelu, A.~White, and H.~van Hasselt.
\newblock Learning expected emphatic traces for deep {RL}.
\newblock \emph{CoRR}, abs/2107.05405, 2021.
\newblock URL \url{https://arxiv.org/abs/2107.05405}.

\bibitem[Ke et~al.(2018)Ke, Goyal, Bilaniuk, Binas, Mozer, Pal, and
  Bengio]{Ke2018SparseAB}
N.~R. Ke, A.~Goyal, O.~Bilaniuk, J.~Binas, M.~Mozer, C.~Pal, and Y.~Bengio.
\newblock Sparse attentive backtracking: Temporal creditassignment through
  reminding.
\newblock In \emph{NeurIPS}, 2018.

\bibitem[Kingma and Ba(2015)]{adam}
D.~P. Kingma and J.~Ba.
\newblock Adam: {A} method for stochastic optimization.
\newblock In Y.~Bengio and Y.~LeCun, editors, \emph{3rd International
  Conference on Learning Representations, {ICLR} 2015, San Diego, CA, USA, May
  7-9, 2015, Conference Track Proceedings}, 2015.
\newblock URL \url{http://arxiv.org/abs/1412.6980}.

\bibitem[Kolter(2011)]{Kolter2011TheFP}
J.~Z. Kolter.
\newblock The fixed points of off-policy td.
\newblock In \emph{NIPS}, 2011.

\bibitem[Kumar et~al.(2020)Kumar, Gupta, and Levine]{Kumar2020DisCorCF}
A.~Kumar, A.~Gupta, and S.~Levine.
\newblock Discor: Corrective feedback in reinforcement learning via
  distribution correction.
\newblock \emph{ArXiv}, abs/2003.07305, 2020.

\bibitem[Maei(2011)]{Maei:2011}
H.~R. Maei.
\newblock \emph{Gradient temporal-difference learning algorithms}.
\newblock PhD thesis, University of Alberta, 2011.

\bibitem[Mahmood et~al.(2015)Mahmood, Yu, White, and
  Sutton]{Mahmood2015EmphaticTL}
A.~R. Mahmood, H.~Yu, M.~White, and R.~S. Sutton.
\newblock Emphatic temporal-difference learning.
\newblock \emph{ArXiv}, abs/1507.01569, 2015.

\bibitem[McMahan and Gordon(2005)]{McMahan2005FastEP}
H.~B. McMahan and G.~J. Gordon.
\newblock Fast exact planning in markov decision processes.
\newblock In \emph{ICAPS}, 2005.

\bibitem[Mesnard et~al.(2020)Mesnard, Weber, Viola, Thakoor, Saade,
  Harutyunyan, Dabney, Stepleton, Heess, Guez, Hutter, Buesing, and
  Munos]{Mesnard2020CounterfactualCA}
T.~Mesnard, T.~Weber, F.~Viola, S.~Thakoor, A.~Saade, A.~Harutyunyan,
  W.~Dabney, T.~Stepleton, N.~Heess, A.~Guez, M.~Hutter, L.~Buesing, and
  R.~Munos.
\newblock Counterfactual credit assignment in model-free reinforcement
  learning.
\newblock \emph{ArXiv}, abs/2011.09464, 2020.

\bibitem[Moore and Atkeson(2004)]{Moore2004PrioritizedSR}
A.~Moore and C.~Atkeson.
\newblock Prioritized sweeping: Reinforcement learning with less data and less
  time.
\newblock \emph{Machine Learning}, 13:\penalty0 103--130, 2004.

\bibitem[Peng and Williams(1993)]{Peng1993EfficientLA}
J.~Peng and R.~J. Williams.
\newblock Efficient learning and planning within the dyna framework.
\newblock \emph{Adaptive Behavior}, 1:\penalty0 437 -- 454, 1993.

\bibitem[Peng and Williams(1996)]{Peng1996IncrementalMQ}
J.~Peng and R.~J. Williams.
\newblock Incremental multi-step q-learning.
\newblock \emph{Machine Learning}, 22:\penalty0 283--290, 1996.

\bibitem[Precup et~al.(2001)Precup, Sutton, and Dasgupta]{precup2001off}
D.~Precup, R.~S. Sutton, and S.~Dasgupta.
\newblock Off-policy temporal-difference learning with function approximation.
\newblock \emph{ICML}, pages 417--424, 2001.

\bibitem[Robbins and Monro(1951)]{robbins1951stochastic}
H.~Robbins and S.~Monro.
\newblock A stochastic approximation method.
\newblock \emph{The annals of mathematical statistics}, 22:\penalty0 400--407,
  1951.

\bibitem[Sutton(1984)]{Sutton1984TemporalCA}
R.~Sutton.
\newblock Temporal credit assignment in reinforcement learning.
\newblock 1984.

\bibitem[Sutton and Singh(1994)]{Sutton1994OnSA}
R.~Sutton and S.~Singh.
\newblock On step-size and bias in temporal-difference learning.
\newblock In \emph{Proceedings of the Eighth Yale Workshop on Adaptive and
  Learning Systems}, pages 91--96. Yale University, New Haven, CT., 1994.

\bibitem[Sutton et~al.(1999)Sutton, Precup, and Singh]{Sutton1999BetweenMA}
R.~Sutton, D.~Precup, and S.~Singh.
\newblock Between mdps and semi-mdps: A framework for temporal abstraction in
  reinforcement learning.
\newblock \emph{Artif. Intell.}, 112:\penalty0 181--211, 1999.

\bibitem[Sutton et~al.(2008)Sutton, Szepesvari, Geramifard, and
  Bowling]{Sutton2008DynaStylePW}
R.~Sutton, C.~Szepesvari, A.~Geramifard, and M.~Bowling.
\newblock Dyna-style planning with linear function approximation and
  prioritized sweeping.
\newblock In \emph{UAI}, 2008.

\bibitem[Sutton(1988{\natexlab{a}})]{Sutton:1988}
R.~S. Sutton.
\newblock Learning to predict by the methods of temporal differences.
\newblock \emph{Machine learning}, 3\penalty0 (1):\penalty0 9--44,
  1988{\natexlab{a}}.

\bibitem[Sutton(1988{\natexlab{b}})]{sutton1988learning}
R.~S. Sutton.
\newblock Learning to predict by the methods of temporal differences.
\newblock \emph{Machine learning}, 3\penalty0 (1):\penalty0 9--44,
  1988{\natexlab{b}}.

\bibitem[Sutton and Barto(2018)]{sutton2018}
R.~S. Sutton and A.~G. Barto.
\newblock \emph{Reinforcement Learning: An Introduction}.
\newblock The MIT Press, Cambridge, MA, 2018.

\bibitem[Sutton et~al.(2014)Sutton, Mahmood, Precup, and {van
  Hasselt}]{Sutton:2014}
R.~S. Sutton, A.~R. Mahmood, D.~Precup, and H.~{van Hasselt}.
\newblock A new {Q}($\lambda$) with interim forward view and {Monte Carlo}
  equivalence.
\newblock In \emph{International Conference on Machine Learning}, pages
  568--576, 2014.

\bibitem[Sutton et~al.(2016)Sutton, Mahmood, and White]{sutton2016emphatic}
R.~S. Sutton, A.~R. Mahmood, and M.~White.
\newblock An emphatic approach to the problem of off-policy temporal-difference
  learning.
\newblock \emph{The Journal of Machine Learning Research}, 17\penalty0
  (1):\penalty0 2603--2631, 2016.

\bibitem[Tsitsiklis and {Van Roy}(1997)]{Tsitsiklis_VanRoy:97}
J.~N. Tsitsiklis and B.~{Van Roy}.
\newblock An analysis of temporal-difference learning with function
  approximation.
\newblock \emph{{IEEE} Transactions on Automatic Control}, 42\penalty0
  (5):\penalty0 674--690, 1997.

\bibitem[{van Hasselt} and Sutton(2015)]{Hasselt2015LearningTP}
H.~{van Hasselt} and R.~Sutton.
\newblock Learning to predict independent of span.
\newblock \emph{ArXiv}, abs/1508.04582, 2015.

\bibitem[van Hasselt et~al.(2018)van Hasselt, Doron, Strub, Hessel, Sonnerat,
  and Modayil]{hasselt2018}
H.~van Hasselt, Y.~Doron, F.~Strub, M.~Hessel, N.~Sonnerat, and J.~Modayil.
\newblock Deep reinforcement learning and the deadly triad.
\newblock \emph{CoRR}, abs/1812.02648, 2018.

\bibitem[van Hasselt et~al.(2019)van Hasselt, Hessel, and
  Aslanides]{Hasselt2019WhenTU}
H.~van Hasselt, M.~Hessel, and J.~Aslanides.
\newblock When to use parametric models in reinforcement learning?
\newblock In \emph{Advances in Neural Information Processing Systems 36,
  NeurIPS}, 2019.

\bibitem[{van Hasselt} et~al.(2020){van Hasselt}, Madjiheurem, Hessel, Silver,
  Barreto, and Borsa]{van2020expected}
H.~{van Hasselt}, S.~Madjiheurem, M.~Hessel, D.~Silver, A.~Barreto, and
  D.~Borsa.
\newblock Expected eligibility traces.
\newblock \emph{arXiv preprint arXiv:2007.01839}, 2020.

\bibitem[van Hasselt et~al.(2021)van Hasselt, Madjiheurem, Hessel, Silver,
  Barreto, and Borsa]{hasselt2020}
H.~van Hasselt, S.~Madjiheurem, M.~Hessel, D.~Silver, A.~Barreto, and D.~Borsa.
\newblock Expected eligibility traces.
\newblock \emph{Proceedings of the AAAI Conference on Artificial Intelligence},
  35\penalty0 (11):\penalty0 9997--10005, May 2021.

\bibitem[{van Seijen} and Sutton(2014)]{vanSeijen:2014}
H.~{van Seijen} and R.~S. Sutton.
\newblock True online {TD}($\lambda$).
\newblock In \emph{International Conference on Machine Learning}, pages
  692--700, 2014.

\bibitem[Watkins and Dayan(1992)]{watkins1992q}
C.~J. Watkins and P.~Dayan.
\newblock Q-learning.
\newblock \emph{Machine learning}, 8\penalty0 (3-4):\penalty0 279--292, 1992.

\bibitem[White(2017)]{white2017unifying}
M.~White.
\newblock Unifying task specification in reinforcement learning.
\newblock In \emph{International Conference on Machine Learning}, pages
  3742--3750. PMLR, 2017.

\bibitem[White and White(2016)]{White2016AGA}
M.~White and A.~White.
\newblock A greedy approach to adapting the trace parameter for temporal
  difference learning.
\newblock \emph{ArXiv}, abs/1607.00446, 2016.

\bibitem[Xu et~al.(2018)Xu, {van Hasselt}, and Silver]{Xu2018MetaGradientRL}
Z.~Xu, H.~{van Hasselt}, and D.~Silver.
\newblock Meta-gradient reinforcement learning.
\newblock \emph{Advances in Neural Information Processing Systems},
  31:\penalty0 2402--2413, 2018.

\bibitem[Zahavy et~al.(2020)Zahavy, Xu, Veeriah, Hessel, Oh, van Hasselt,
  Silver, and Singh]{zahavy2020self}
T.~Zahavy, Z.~Xu, V.~Veeriah, M.~Hessel, J.~Oh, H.~P. van Hasselt, D.~Silver,
  and S.~Singh.
\newblock A self-tuning actor-critic algorithm.
\newblock \emph{Advances in Neural Information Processing Systems}, 33, 2020.

\bibitem[Zhang et~al.(2020)Zhang, Veeriah, and Whiteson]{zhang2020learning}
S.~Zhang, V.~Veeriah, and S.~Whiteson.
\newblock Learning retrospective knowledge with reverse reinforcement learning.
\newblock In \emph{Advances in Neural Information Processing Systems},
  volume~33, 2020.

\end{thebibliography}
